\documentclass{article}

\usepackage{microtype}
\usepackage{graphicx}
\usepackage{subfigure}
\usepackage{booktabs} %

\usepackage{hyperref}
\usepackage[table, dvipsnames]{xcolor}
\usepackage{soul}

\PassOptionsToPackage{numbers, square}{natbib}
\usepackage[main,final]{neurips_2026}

\usepackage{amsmath}
\usepackage{amssymb}
\usepackage{mathtools}
\usepackage{amsthm}

\usepackage{algorithm, algorithmic}
\usepackage[capitalize,noabbrev]{cleveref}
\theoremstyle{plain}
\newtheorem{theorem}{Theorem}[section]
\newtheorem{proposition}[theorem]{Proposition}
\newtheorem{lemma}[theorem]{Lemma}
\newtheorem{corollary}[theorem]{Corollary}
\theoremstyle{definition}
\newtheorem{definition}[theorem]{Definition}
\newtheorem{assumption}[theorem]{Assumption}
\theoremstyle{remark}

\usepackage[textsize=tiny]{todonotes}

\usepackage{url}
\usepackage{wrapfig}

\newcommand{\edit}[1]{\textbf{\color{blue}}}

\crefname{section}{Sec.}{Secs.}
\crefname{appendix}{App.}{Apps.}
\crefname{equation}{Eq.}{Eqs.}
\crefname{figure}{Fig.}{Figs.}
\crefname{proposition}{{Prop.}}{Props.}
\crefname{corollary}{Cor.}{Cors.}
\crefname{algorithm}{Alg.}{Algs.}

\usepackage{pifont}

\usepackage{xspace}
\usepackage{enumitem}
\usepackage{multirow}
\usepackage{wrapfig}  
\newcommand{\algname}{UMPIRE\xspace}

\newcommand{\cu}{NC\xspace}
\newcommand{\lne}{LN-Ent\xspace}
\newcommand{\se}{Sem.Ent\xspace}
\newcommand{\eigen}{Eigen\xspace}

\newcommand{\task}{t}
\newcommand{\taskspace}{\mathcal{T}}
\newcommand{\img}{I}
\newcommand{\query}{q}
\newcommand{\ans}{y}

\newcommand{\ansset}{\mathcal{Y}}
\newcommand{\model}{\mathcal{M}}
\newcommand{\metric}{v}
\newcommand{\method}{V}
\newcommand{\unadjmethod}{U}
\newcommand{\quadent}{Q}
\newcommand{\acc}{a}

\newcommand{\inc}{c}

\newcommand{\llava}{Llava-v1.5-13b~}
\newcommand{\modelphi}{Phi-4-multimodal-instruct~}
\newcommand{\llavavideo}{LLaVA-NeXT-Video-7b-hf~}

\usepackage{makecell}

\usepackage{caption}

\begin{document}

\title{UMPIRE: Uncertainty Quantification for \\ Multimodal Large Language Models with \\ Incoherence-adjusted Semantic Volume}

\author{%
  Gregory Kang Ruey Lau$^{1,2}$\thanks{Equal contribution.}
  \quad
  Hieu Dao$^{1}$\footnotemark[1]
  \quad
  Nicole Kan Hui Lin$^{1}$
  \quad
  Bryan Kian Hsiang Low$^{1}$
  \\
  $^1$ Department of Computer Science, National University of Singapore, Singapore \\
  $^2$ CNRS@CREATE, 1 Create Way, \#08-01 Create Tower, Singapore 138602 \\
  \texttt{\{gregorylau,daohieu\}@u.nus.edu}, \texttt{dcslkh@nus.edu.sg}
}

\maketitle

\begin{abstract}

Despite their capabilities, Multimodal Large Language Models (MLLMs) may produce plausible but erroneous outputs, hindering reliable deployment. Accurate uncertainty metrics could enable escalation of unreliable queries to human experts or larger models for improved performance. However, existing uncertainty metrics have practical constraints, such as being designed only for specific modalities, reliant on external tools, or computationally expensive. We introduce \algname, a training-free uncertainty quantification framework for MLLMs that works efficiently across various input and output modalities without external tools, relying only on the models' own internal modality features. Inspired by Determinantal Point Processes (DPPs), \algname computes a novel incoherence-adjusted semantic volume of sampled MLLM responses for a given task instance, which couples semantic embedding dispersion with the quadratic entropy of the responses. This approach effectively captures both the global semantic diversity of samples and the local incoherence of responses based on internal model confidence. We provide theoretical analysis motivating \algname's design, and extensive experiments show that \algname consistently outperforms baseline metrics in error detection and uncertainty calibration across image, audio, and video-text benchmarks, even with small sample budgets. We also demonstrate \algname's generalization to non-text output tasks, including image and audio generation.
Our code is available at \url{https://github.com/daohieu17ctt/UMPIRE}.

\end{abstract}

\section{Introduction}

The capabilities of Multimodal Large Language Models (MLLMs) have expanded rapidly, going beyond image-text tasks to process diverse input modalities such as audio, video and more \citep{liu2023llava,hartsockVisionlanguageModelsMedical2024, yin2024survey}. However, reliable deployment of these models in high-stakes practical settings (e.g., medical analysis \citep{LiuMedical2023, TianMedical2024, LeeMedical2025}) remains challenging due to their tendency to produce plausible but erroneous outputs, or confabulations \citep{berrios1998confabulations, farquharDetectingHallucinationsLarge2024}, potentially more so than text-only LLMs given greater complexities of processing multimodal input due to issues like cross-modal grounding failures or multimodal reasoning-access gap \citep{tiger2026,liang2022mindthegap}. While some works have attempted to directly mitigate these issues during model training by adjusting the training data \citep{gaiveliu2023aligning, yu2024hallucidoctor, wang2024mitigating, yue2024less}, model architecture \citep{liu2024improved, tong2024eyes, zhai2023halle}, or training process \citep{jiangHallucinationAugmentedContrastive2024, yue2024less}, these errors cannot be fully eliminated given noisy and ambiguous real-world data.

A complementary approach to tackling this challenge is developing effective uncertainty quantification methods for MLLMs, which would enable escalation of task instances with unreliable model responses to human experts or larger models.
However, existing uncertainty quantification works largely focus on text-only LLM settings \citep{kuhnSemanticUncertaintyLinguistic2023, malinin2021uncertainty} that may not capture multimodal coherence signals (e.g., text output ungrounded in the input media), or are designed only for specific modalities (e.g., image-text input) such as by relying on external verifiers \citep{gaiveliu2023aligning, mhalsun2023aligning}, or requiring relatively expensive modality-specific feature engineering\citep{zhangVLUncertaintyDetectingHallucination2024, khan2024consistency} that may not be practical especially in settings with resource constraints. Given our inherently multimodal environment, rather than having to specially engineer separate modality-specific uncertainty metrics for new data modalities \citep{yin2024survey, feng2025smellnet} that may arise, which is not scalable nor effective, we ask: 
\emph{Can we achieve an effective training-free MLLM uncertainty framework that can generalize across modalities without requiring modality-specific engineering?}

In our work, we present \algname (\underline{U}ncertainty using \underline{M}odel \underline{P}robability \underline{I}ndicators and \underline{R}esponse \underline{E}mbeddings), a training-free framework to estimate the uncertainty of MLLM output by considering both global semantic diversity and local multimodal incoherence of model responses.   
Intuitively, for task instances that the model is uncertain about, its sampled responses would likely be more diverse in semantic space (high volume) and potentially more incoherent with respect to its multimodal input.
Inspired by the quality-diversity decomposition in Determinantal Point Processes (DPPs) \citep{kuleszaDeterminantalPointProcesses2012}, \algname introduces a novel task-instance uncertainty metric based on the incoherence-adjusted semantic volume spanned by sampled MLLM responses, unifying two complementary uncertainty signals: global semantic dispersion and local response incoherence estimated with quadratic entropy. 
\algname makes use of MLLMs' own conditional probabilities as indicators of multimodal coherence and rich multimodal feature space \citep{girdhar2023imagebind, lau2025readme} for semantic embeddings, removing the need for external tools.
This approach allows the framework to generalize natively across various modalities (e.g., image, audio, video) by leveraging the MLLM's own internal multimodal capabilities. Our key contributions are as follows: 
\begin{itemize}[leftmargin=*, nosep]
    \item We develop a novel training-free uncertainty metric based on incoherence-adjusted semantic volume, taking into account both global semantic diversity and local response quality via a quality-diversity kernel inspired by DPP (\Cref{sec:kernel}); 
    \item We provide analysis on how \algname decomposes into an interplay of two complementary uncertainty signals: a semantic volume term and a novel Monte Carlo estimate of quadratic entropy, resulting in \algname's strong performance (\Cref{sec:prac_consi}); 
    \item We show that \algname consistently outperforms baselines in evaluations (e.g., on AUROC, ECE, AURAC) across diverse multimodal input QA benchmarks (e.g., image, audio, video input), and also non-text output tasks (e.g., image, audio generation) despite not needing modality-specific mechanisms (\cref{sec:auroc},~\ref{sec:calibration},~\ref{sec:design_desiderata}). This includes settings involving blackbox API-access MLLM models via whitebox proxy estimation (\cref{sec:effective_practical}).
\end{itemize}

\section{Problem formulation and desiderata}
\label{sec:prob_desi}

\textbf{Problem formulation.} Consider a whitebox MLLM $\model$ that takes in  multi-modality input $\query$ (e.g. a combination of text and image, audio or video input), and autoregressively produces text output $\ans=[w_i]_{i=1}^N$ that are sequences of tokens $w$ from the MLLM decoder's vocab space\footnote{We include analysis and empirical results of other modality output such as images and audio in \cref{sec:design_desiderata}}. 
These MLLMs can be represented as model-generated conditional probability distributions $\mathbb{P}_\model(y|\query)=\mathbb{P}_\model(w_1| \query)\mathbb{P}_\model(w_2|\query, w_1)\ldots \mathbb{P}_\model(w_n|\query, w_{1:N-1})$. 

For each task $\taskspace$, we have task instances $\task \coloneqq (\query_\task; \ans_\task^*) \in \taskspace$ consisting of  
the multimodal input query $\query_\task$, and ground truth output $\ans_\task^*$ that is only available during evaluation. 
We consider an MLLM's response $\hat{\ans}_\task$ to a task instance $\task$ as its most likely output sequence, approximately sampled with low temperature (e.g. $T=0.01$) from $\model$ given query $\query_\task$, and its task instance accuracy as 
a binary indicator $a(\model, \task) \coloneqq \mathbb{I}\{\hat{\ans}_\task=\ans_\task^*\}$ that evaluates how well the response matches the ground truth\footnote{Empirical results are robust to choice of common indicators (\cref{app:eval_ablation}). For image/audio generation tasks without standard binary indicators, we use continuous quality metrics (see \cref{sec:design_desiderata}).}.

Our goal is to develop a framework that computes a task instance-specific uncertainty metric $\metric(\model, \task)$ for any $\task \in \taskspace$ at inference time that is indicative of task instance accuracy $\acc(\model,\task)$. Specifically, we are developing metrics that can be used to indicate confabulations, which is when MLLMs fluently provide arbitrary, wrong responses to queries due to lack of relevant knowledge
\citep{berrios1998confabulations, farquharDetectingHallucinationsLarge2024}. Intuitively, the assumption is that the more uncertain the MLLM is about what the correct answer is, the more diverse its sampled (wrong) responses would be. This behavior will lead to sampled response diversity that we can characterize via a metric $\metric(\model, \task)$ without access to the ground truth output.
We do not consider other inaccuracies such as when MLLMs consistently produce the same wrong responses to queries when sampled (e.g., due to erroneous data).
However, confabulations are common and a good metric $\metric(\model, \task)$ that detect these errors could also be indicative of task instance accuracy $\acc(\model, \task)$ as past works have demonstrated \citep{farquharDetectingHallucinationsLarge2024}, and which we also show empirically (\cref{sec:exp}).

\label{sec:desiderata}
\textbf{Effectiveness desiderata.} We consider two desiderata covering key functions that the metric could play in practice: discriminating unreliable instances and providing a continuous notion of risk. 
\newcounter{maincount} 
\newcounter{subs}[maincount] %
\renewcommand{\thesubs}{\arabic{maincount}\alph{subs}}
\begin{enumerate}[label=\textbf{D\arabic*},leftmargin=*,align=left,itemsep=5pt,topsep=0pt,parsep=0pt,partopsep=0pt, start=1,wide=0pt]

\item \label{r:classify}\textbf{Discrimination.} The metric should effectively distinguish between task instances that the MLLM will get correct or not. Formally, let $\mathcal{C}\coloneqq\{\task \mid \acc(\model,\task)=1\}$ and $\mathcal{W}\coloneqq\{\task \mid \acc(\model,\task)=0\}$ be the sets of correct and wrong instances. For randomly sampled pairs $\task_c \in \mathcal{C}$ and $\task_w \in \mathcal{W}$,
\begin{equation}
    \label{eq:distinguish}
    \mathbb{P}[\metric(\model,\task_w) > \metric(\model,\task_c)] \approx 1.
\end{equation}
Ideally, the left-hand side expression of \Cref{eq:distinguish} (or equivalently,
AUROC) approaches 1, meaning there exists a threshold $\gamma$ that can reliably classify errors (i.e., $\metric > \gamma \implies \task \in \mathcal{W}$). In practice, the achievable value depends on the task and model, but the higher the AUROC, the better the metric can discriminate incorrect responses.

\item \label{r:risk} \textbf{Risk-score quality.}
Beyond \ref{r:classify}, the metric should be a \emph{continuous proxy} for instance-level error risk, supporting a simple mapping to model error probabilities assessable via two sub-criteria:

\begin{enumerate}[label=\textbf{D\arabic{enumi}\alph*}, ref=D\arabic{enumi}\alph*, leftmargin=0pt, labelindent=0pt, wide=0pt, itemsep=2pt]

\item \label{r:r2a} \textbf{Risk-linearity.}
The metric should be linearly correlated with the model's conditional error probability so that it is directly usable for risk control:
\begin{equation}
\label{eq:r2a_linearity_clean}
\mathbb{P}(\acc(\model,\task)=0 \mid \metric=s) \approx \beta_0 + \beta_1 s, \quad \beta_1 > 0.
\end{equation}
We evaluate this via Pearson correlation (CPC) based on standard reliability-curve binning (\cref{sec:calibration}), which measures risk-linearity rather than probabilistic calibration.
Unlike \ref{r:classify}, this enables relative comparisons, e.g., users can adjust rejection thresholds anticipating proportionate impact on expected error rates without normalization (unlike \ref{r:r2b}).
\item \label{r:r2b} \textbf{Unlabeled normalization to a probability proxy.}
Building on \ref{r:r2a}, given only \emph{unlabeled} task instances, the score should be easily normalizable to $\tilde{\metric}\in[0,1]$ (e.g., via min-max scaling) such that it empirically tracks the approximate error probability, i.e.,
\begin{equation}
\label{eq:r2b_calibration_clean}
\mathbb{P}(\acc(\model,\task)=0\mid \tilde{\metric}=p) \approx p,\quad \forall p\in[0,1].
\end{equation}
We evaluate this via the Expected Calibration Error (ECE) of the normalized score $\tilde{\metric}$ \citep{guoCalibrationModernNeural2017b}. Since $\tilde{\metric}$ is obtained without labels, it is a proxy for error probability rather than a calibrated posterior. When a small labeled set is available, standard post-hoc calibration (e.g., Platt scaling) can be applied (\cref{app:platt}). Note that \ref{r:risk} complements \ref{r:classify}, as a metric can be well-calibrated but poorly discriminative (see \Cref{app:effectdesiderata}).

\end{enumerate}

\end{enumerate}

We also consider design guidelines related to common practical deployment requirements that can serve as axes for metric evaluation, though their thresholds and prioritization will depend on use case:
\begin{enumerate}[label=\textbf{G\arabic*},leftmargin=*,itemsep=5pt,topsep=0pt,parsep=0pt,partopsep=0pt, start=1,wide=0pt]

\item \label{r:generalizability} \textbf{Multimodal generalizability.} Metric that are directly usable for MLLMs of different input modalities (e.g., text, image, audio, video) and still satisfy all desiderata, are preferred. A stricter version, \ref{r:generalizability}', is for the metric to apply across different output modalities (e.g., image/audio generation) as well. While many MLLM works consider only multimodal input with text output, we provide some analysis on \ref{r:generalizability}' in \Cref{sec:design_desiderata}. Evaluation involves assessing the metric for \ref{r:classify}, \ref{r:risk} across modalities.

\item \label{r:coherence} \textbf{Multimodal coherence.} Metric that consider the coherence of each sampled response with respect to all input modalities in the task instance query (e.g., whether the response is grounded in all input modalities) are preferred. We test this by recomputing the uncertainty metrics after removing/corrupting one modality while holding the generated response fixed -- a multimodal coherent metric should degrade predictably (see \cref{sec:method} and \cref{app:coherence_exp} for specific evaluations).

\item \label{r:efficiency} \textbf{Computational Efficiency.} Metric that (a) have fast computational runtime are preferred. For some constrained inference settings, metrics that (b) do not require external tools or reward models (i.e., only need inference on the single MLLM) may also be preferred.

\end{enumerate}

\begin{figure*}[t]
    \centering
    \includegraphics[width=\linewidth]{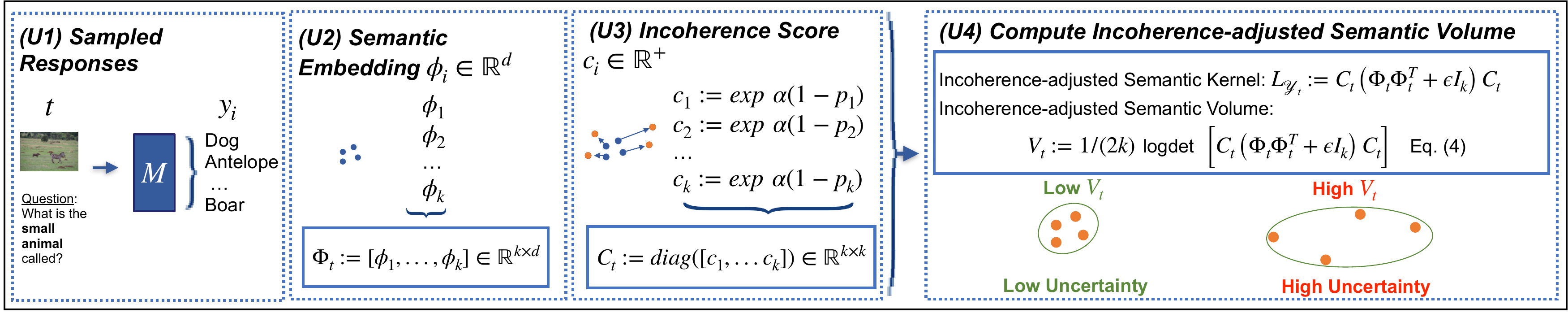}
    \caption{Schematic describing the \algname framework.} 
    \label{fig:schematic}
\end{figure*}

\section{Method}
\label{sec:method}

Our proposed framework, \algname, combines several insights to meet the desiderata. First, well-trained MLLMs that can handle multimodal tasks would already possess rich modality-specific features.
Hence, we propose developing a \textit{modality-general framework that relies on an MLLM's own rich inherent embeddings, removing the need for the use of external tools} (\ref{r:generalizability}, \ref{r:efficiency}). Second, MLLMs can \textit{take into account all input modalities and response incoherence via their model-generated conditional probabilities} (\ref{r:coherence}), which though uncalibrated contain useful signals on the relative multimodal coherence of responses.
Third, model uncertainty about the correct answer can manifest as more spread-out output distributions (\cref{sec:prob_desi}) and diverse output responses when sampled. This motivates us to consider whether \textit{diversity measures} could serve as a basis for our uncertainty metric.

\subsection{Quality-diversity kernel and semantic volume}
\label{sec:kernel}
Hence, inspired by determinantal point processes (DPP) literature \citep{kuleszaDeterminantalPointProcesses2012, lau2024dipper}, we characterize uncertainty via a \emph{quality-diversity} kernel applied to sampled responses in the MLLM's embedding/semantic space, which captures both a local measure of quality for each response (based on the model's own incoherence signal) and a global measure of how semantically different the responses are from each other. Intuitively, given samples of $k$ responses represented on a hypersphere in semantic space (e.g., based on normalized MLLM embeddings), samples with greater response diversity have more widely spaced out points that enclose larger `semantic volume'. Meanwhile, each response would also have varying degrees of `quality', based on how incoherent the MLLM assesses it to be (e.g., model output probabilities), that we can use to scale its norm. More incoherent responses will then have bigger norms and contribute to larger enclosed volume, leading to greater uncertainty. 
Specifically, our framework consists of the following steps (\Cref{fig:schematic}):

\begin{enumerate}[label=\textbf{U\arabic*},leftmargin=*, itemindent=0pt, align=left, listparindent=0pt, itemsep=5pt,topsep=0pt,parsep=0pt,partopsep=0pt, start=1, wide=0pt]
\item \textbf{Sampling.} Given a task instance $\task \in \taskspace$, we sample $k$ MLLM responses 
$\ansset_\task=[\ans_i]_{i=1}^{k}$ by standard MLLM inference settings such as temp $T=1$ and nucleus sampling (see ablations in \cref{app:ablation_study}).

\item \textbf{Semantic embedding.} For each response $y_i$, we extract the normalized, last MLLM embedding layer vector of the EOS token at the end of the response (ablations in \Cref{app:layers}) 
$\phi_i \in \mathbb{R}^d$, $d>k$.
The $k$ samples form 
a $k \times d$ embedding matrix $\Phi_\task$. The MLLM's rich multimodal semantic embeddings satisfy \ref{r:generalizability}: the $k$ samples lie on a $d-1$ dimensional hypersphere in the relevant modalities' semantic space, with the angular dispersion among the responses representing semantic distances \citep{reimers-2019-sentence-bert}.

\item \textbf{Incoherence score.}
\label{u:incoherence_score} 
Concurrently, we compute an incoherence score $\inc_i \in \mathbb{R}^+$ for each response, where larger $c_i$ indicates contribution to a higher uncertainty score. In general, we can consider scores of the form $c_i = \exp (\alpha f_i(\model,\task))$, where $f_i(\model,\task)$ is a response-specific incoherence value conditional on all input modalities (\ref{r:coherence}) and $\alpha \geq 0$ is a task $\taskspace$ specific constant that sets the scale for $f_i$. 
We propose setting $f_i(\model,\task)\coloneq (1-p_i)$, where $p_i$ is the $\model$ model-generated probability for each response $y_i$ conditional on task input $\task$, i.e.,
$p_i := \exp( \sum_{j=1}^{N_i}\log \mathbb{P}_\model\left(w_{i,j}\mid q_t, w_{i,1:j-1}\right) )= \mathbb{P}_\model(y_i \mid \query_t)$. As a well-trained MLLM that is useful in practice will have its response distribution $\mathbb{P}_\model(y_i \mid \query_t)$ conditional on all multimodal components of query $\query_\task$, taking $p_i$ into account contributes to satisfying multimodal coherence \ref{r:coherence} without external verifiers, which analyze empirically later in \cref{sec:exp}. We discuss response length normalized variants for long-form generation in \cref{sec:prac_consi} and \cref{sec:ln_prob}.
We chose the form $f_i=(1-p_i)$ since: (1) it can be interpreted as the model's internal measure of its relative doubt about response $i$; and (2) it leads to \algname incorporating an interpretable measure of the sampled responses' \emph{quadratic entropy}, which we discuss in \cref{sec:prac_consi}.
With $k$ samples, we can define a $k \times k$ incoherence scaling matrix $C_\task \coloneq \text{diag}(\exp(\alpha (1-p_i)))$ that scales the norm of each response embedding by its incoherence score.

\item \textbf{Incoherence-adjusted semantic volume.} Without any external tools (\ref{r:efficiency}), 
we can compute the incoherence-adjusted semantic kernel based on the Gram matrix formed by the semantic embeddings, adjusted by the incoherence scores: 
$L_{\ansset_\task} \coloneqq C_\task (\Phi_\task\Phi_\task^\top+ \epsilon I_k) C_\task$, 
where $\epsilon > 0$ is a small jitter and $I_k$ a $k \times k$ identity matrix, added for numerical stability and to avoid degeneracy.
To compute an uncertainty score, note that from geometry, $\det (L_{\ansset_\task})$ is a scalar representing
the squared volume spanned by the Gram matrix vectors. Hence we define the \algname uncertainty metric:
\begin{equation}
\label{eq:umpire}
\method_\task\coloneqq 1/(2k)  \log \det \left [ C_\task \left (\Phi_\task\Phi_\task^\top+ \epsilon I_k \right ) C_\task \right ],
\end{equation}
\end{enumerate}
which is the sample count normalized log of the incoherence-adjusted semantic volume, with a DPP-style quality-diversity kernel over sampled responses. In our context, the `quality' term is the incoherence score
which is a function of model-generated response probabilities, while `diversity' is measured in the MLLM's response embedding space.
Samples with more spread-out responses, each with higher incoherence scores, will end up with larger incoherence-adjusted semantic volume, representing more diversity and larger \algname uncertainty values.

\subsection{Analysis and practical considerations}
\label{sec:prac_consi}
While the DPP quality-diversity kernel in \cref{eq:umpire} provides the geometric foundations for the exponential quality (incoherence score) scaling and Gram matrix diversity, its determinant structure enables us to also decompose the \algname metric into two complementary terms (\cref{prop:app:decomp}, \cref{app:theory_setup}) that allows for further interpretability and more numerically stable implementation:
\begin{align}
    \label{eq:casev}
    V_t = \tfrac{1}{2k}  \log \det (\Phi_\task\Phi_\task^\top+ \epsilon I_k) + \tfrac{\alpha}{k} \textstyle \sum_{i=1}^k (1 - p_i), 
\end{align}
where $\alpha$ is the sole hyperparameter balancing both terms, which we can heuristically set without labeled data (details below). The first term $\unadjmethod_\task\coloneqq 1/(2k)  \log \det (\Phi_\task\Phi_\task^\top+ \epsilon I_k)$ is the unadjusted semantic volume, which captures the semantic diversity among the sampled responses. Samples with responses that comprise a wide range of semantic meanings (e.g., many plausible answers) will have larger volumes, while those where responses have similar meaning will result in small volumes (the $\epsilon I_k$ regularization keeps the logdet term well defined).   
To provide more intuition, we provide theoretical analysis on how $\unadjmethod_\task$ increases as semantic modes become more separate, under a mixture of modes assumption as illustration (\cref{prop:app:psdmono}, \cref{app:theory:mixture}). This relates to how more diverse, competing semantic responses result in higher unadjusted semantic volume, indicating greater uncertainty.

The second term, $\quadent_\task \coloneq (1/k)\sum_i^k f_i(\model,\task)$, can be interpreted as a Monte Carlo estimate of the expected incoherence value $\mathbb{E}[f(Y)]$ when $Y\stackrel{iid}{\sim} \mathbb{P}_\model(\cdot| t)$. In particular, under our default choice $f_i=(1-p_i)$, which can be reparameterized as $f(y)=1-\mathbb{P}_\model(y\mid \query_t)$ for each response from $Y\stackrel{iid}{\sim}\mathbb{P}_\model(\cdot|t)$ and $p_i=\mathbb{P}_\model(y_i|t)$,
$\quadent_\task$ serves as a novel likelihood-based uncertainty signal: the Monte Carlo estimate of quadratic entropy 
$\mathbb{E}_{Y\sim \mathbb{P}_\model(\cdot|t)}\!\big[1-\mathbb{P}_\model(Y|t)\big]
\ = 1-\sum_y \mathbb{P}_\model(y\mid \query_t)^2$
for a given task instance $t$ (see \cref{lem:app:QasH2}, \cref{app:theory:quadratic}). The quadratic entropy, which had not been used in past MLLM uncertainty works, captures the dispersion of probability mass over responses and can be interpreted as the probability that two sampled responses from $\mathbb{P}(\cdot|t)$ do not match \citep{vajda1968bounds,Rnyi1961OnMO, vajda2007generalized}, effectively indicating higher uncertainty when a model spreads out probability mass across different responses. 
Note that while $\quadent_\task$ does not explicitly consider semantic distances, observing small $\quadent_\task$ values constraints the diversity of semantic modes present in the sampled responses as it serves as an upper bound (see \cref{lem:app:coarsen}, \cref{app:smallQ-modes}), implying that probability mass is largely concentrated in a dominant/small number of semantic modes (\cref{cor:app:dominant}, \cref{cor:app:manymodes}).
Unlike the more commonly used Shannon entropy, quadratic entropy is less sensitive to long tails of very small probabilities (i.e., in $\log p$) of model responses and typically yield lower variance estimates for small $k$. Depending on use cases, our framework allows for various forms of $f_i$ that may have different sensitivities to $p_i$ and other characteristics.

\textbf{Interplay.} Note that $\unadjmethod_\task$ and $\quadent_\task$ are complementary to each other. The term $\quadent_\task$ is often a strong discriminator (\ref{r:classify}) as probability mass dispersion generally rises when the model is uncertain, and as mentioned small $\quadent_\task$ provides some constraints on semantic diversity. However, it is unable to distinguish different sources of dispersion (e.g., large semantic ambiguity or predominantly lexical variability among sampled responses). The term $\unadjmethod_\task$ resolves this as it measures \emph{semantic} spread among responses, differentiating among regimes where there are more semantic modes or greater semantic variance even when they have similar probability mass dispersions. Empirically, this additional information improves risk-score quality (\ref{r:risk}), allowing the overall \algname metric $\method_\task$ to better satisfy the desiderata with more linear reliability curves (higher CPC) and min-max normalized scores with lower ECE, while maintaining competitive AUROC (\cref{sec:exp}). We see this across most datasets by examining the separate performances of $\quadent_\task$, $\unadjmethod_\task$ and $\method_\task$ (\cref{tab:adaptive_best}, \cref{app:alphasensitivity}), and have also performed Likelihood Ratio Tests demonstrating the statistical significance of $\unadjmethod_\task$ in complementing $\quadent_\task$ to better estimate model error probability (\cref{tab:lrt_results}). To get better intuition on why \ref{r:classify}/AUROC performance can be maintained, note that as a rank-based metric, it is robust to monotonic transformations of the underlying scores. Since $\unadjmethod_\task$ and $\quadent_\task$ are often positively correlated (increasing with uncertainty), the combined metric $\method_\task$ would largely preserve the ranking of correct versus incorrect instances established by $\quadent_\task$, while refining the \emph{magnitude} of the scores to better align with error probabilities (\ref{r:risk}). \cref{app:qualitative} shows qualitative examples illustrating how considering only one term could lead to sub-optimal outcomes.

The hyperparameter $\alpha$ balances the contribution of both terms and is set for each overall task $\taskspace$. We found that a principled and simple heuristic for setting $\alpha$ \emph{without labeled data} (which we denote as adaptive $\alpha$) is to match the scales of both terms by setting $\alpha$ to be the ratio of the median of $\unadjmethod_\task$ and $\quadent_\task$ based on an unlabeled subset of task instances. This is usually sufficient for good empirical results as seen in our experiments in \cref{sec:exp}. Alternatively, having a small labeled dev set could allow hyperparameter tuning that can further boost results (\Cref{app:alphasensitivity}, \cref{tab:adaptive_best}).

\textbf{Practical considerations.} 
Our \algname metric $\method_\task$ exhibits sub-Gaussian concentration, converging to its conditional mean at $O(1/\sqrt{k})$, and have misranking probabilities decay exponentially with $k\Delta^2$ if correct-wrong instances have an expected-score gap $\Delta$, as we show in \cref{prop:app:concentration}, \cref{app:theory:concentration}. This supports \algname's stable discrimination (\ref{r:classify}) performance with modest sampling budgets, which we observe empirically in \cref{sec:exp} and \cref{app:efficiency} \cref{fig:run_time_num_gen}b.

In practice, \cref{eq:casev} provides a more numerically stable method for computing the \algname metric, where we avoid explicitly forming $C(\cdot)C$ and compute the log determinant via Cholesky decomposition. While the latter has a complexity of $O(k^3)$, this cost is negligible in practice, given the modest $k$ used, compared to the cost of MLLM inference which is incurred by all metrics. We empirically show how \algname is computationally efficient compared to baselines in \cref{sec:design_desiderata}.

\textbf{Length normalization and long-form generation.} 
Similar to past works~\citep{chenLLMsInternalStates2024,kuhnSemanticUncertaintyLinguistic2023,malinin2021uncertainty,khan2024consistency}, we focus on benchmarking \algname on shorter generation tasks such as visual question answering. 
For such short-answer tasks, we recommend using raw sequence probability $p_i$ which preserves our exact quadratic entropy interpretation of $\quadent_\task$, matching our theoretical analysis. For long generations such as visual reasoning, $p_i$ becomes vanishingly small. In these settings, we recommend using length-normalized probability $\bar p_i := p_i^{1/N_i}$ or restrict scoring to a limited answer span (e.g., final extracted answer) similar to past works~\citep{farquharDetectingHallucinationsLarge2024,kuhnSemanticUncertaintyLinguistic2023,malinin2021uncertainty,zhang2025cot}. Since $\bar p_i\in(0,1]$, $Q_t$ remains a bounded, unbiased estimate of $\mathbb{E}[1-\bar p(Y)]$, so the decomposition (\cref{prop:app:decomp}), the analysis of $U_t$, and the concentration bound (\cref{prop:app:concentration}) still hold. On long chain-of-thought tasks (MathVista, MMMU--Computer Science averaging 1{,}740 output tokens), length-normalized \algname continues to outperform all baselines (\cref{tab:long-form}, \cref{app:long-form}).

\section{Experimental results}
\label{sec:exp}

We empirically evaluate whether \algname satisfies the desiderata (\Cref{sec:desiderata}),
primarily conducting experiments for the image-text input setting given its wider range of established benchmarks, but also analyzing audio-text and video-text input modality tasks
to assess 
\ref{r:generalizability}, as well as image and audio generation tasks to assess \ref{r:generalizability}'.
Details of all datasets are in \Cref{app:dataset}.

We use \llava \citep{liu2023llava}, \modelphi \citep{abouelenin2025phi}, \llavavideo \citep{zhang2024llavanextvideo} for image-text, audio-text, and video-text experiments, respectively,
and compare \algname against baselines representative of different approaches, including a modality-specific (image) metric Neighborhood Consistency (\cu) \citep{khan2024consistency}, and text-only LLM uncertainty metrics that we adapt to the multimodal input setting:  LN-Entropy (\lne) \citep{malinin2021uncertainty}, Semantic Entropy (\se) \citep{kuhnSemanticUncertaintyLinguistic2023}, and Eigenscore (\eigen) \citep{chenLLMsInternalStates2024}.
While \cu and \se use external tools and hence are not preferred under \ref{r:efficiency}, we continued to evaluate them on other aspects. Details on experiment settings, more baseline methods and ablation studies to highlight \algname's robustness across parameters are in \Cref{app:exp_setup_result} and \Cref{app:ablation_study}.

\begin{table}[ht]
\caption{Effectiveness of various uncertainty metrics (details in \Cref{app:benchmark}). The evaluation covers (i) AUROC ($\times 100$) ($\uparrow$ better) (\ref{r:classify}), (ii) CPC ($\times 100$) ($\uparrow$ better) (\textbf{\ref{r:r2a}}), (iii) ECE ($\downarrow$ better) (\textbf{\ref{r:r2b}}), and (iv) AURAC ($\times 100$) ($\uparrow$ better) (\cref{sec:aurac}).
Overall, \algname achieves the best or second-best performance across all baselines and modalities, with only marginal gaps when not ranked first. \cref{tab:statistical-tests} in \cref{app:statistical-t-test} provides supplementary results showing that \algname's performance gains over baselines are statistically significant.}
\label{tab:r12_results_rotate}
\centering
\setlength{\tabcolsep}{2.5pt}
\resizebox{\linewidth}{!}{%
\begin{tabular}{ll ccccc >{\columncolor[gray]{0.9}}c cc c >{\columncolor[gray]{0.9}}c}
\toprule
 & & \multicolumn{6}{c}{\textbf{Image}} & \multicolumn{2}{c}{\textbf{Audio}} & \textbf{Video} & \multicolumn{1}{c}{} \\ %
\cmidrule(lr){3-8} \cmidrule(lr){9-10} \cmidrule(lr){11-11}
 &  & VQAv2 & OKVQA & AdVQA & MathVista & VQA-RAD & \textbf{Avg.} & SLUE & SpokenSQ. & VidMME & \multicolumn{1}{c}{\multirow{-2}{*}{\textbf{Avg. [All]}}} \\ %
\midrule

\multirow{5}{*}{\rotatebox[origin=c]{90}{\textbf{AUROC $\uparrow$}}} 
 & NC & 75.1{\scriptsize \textsubscript{±0.1}} & 52.8{\scriptsize \textsubscript{±0.7}} & 64.8{\scriptsize \textsubscript{±0.6}} & 74.6{\scriptsize \textsubscript{±0.2}} & 68.6{\scriptsize \textsubscript{±1.5}} & 67.2 & - & - & - & - \\
 & \lne & 77.0{\scriptsize \textsubscript{±0.2}} & 69.5{\scriptsize \textsubscript{±0.2}} & 64.0{\scriptsize \textsubscript{±0.3}} & 65.8{\scriptsize \textsubscript{±0.3}} & 62.0{\scriptsize \textsubscript{±1.1}} & 67.7 & 88.2{\scriptsize \textsubscript{±0.5}} & 76.6{\scriptsize \textsubscript{±0.2}} & 69.4{\scriptsize \textsubscript{±1.5}} & 71.6 \\
 & \se & 83.5{\scriptsize \textsubscript{±0.3}} & 70.9{\scriptsize \textsubscript{±0.1}} & 75.9{\scriptsize \textsubscript{±0.2}} & 80.7{\scriptsize \textsubscript{±0.7}} & 75.8{\scriptsize \textsubscript{±0.5}} & 77.3 & \underline{89.1{\scriptsize \textsubscript{±1.2}}} & \underline{81.0{\scriptsize \textsubscript{±0.3}}} & 73.1{\scriptsize \textsubscript{±1.5}} & 78.7 \\
 & Eigen & \underline{85.3{\scriptsize \textsubscript{±0.2}}} & \underline{73.6{\scriptsize \textsubscript{±0.1}}} & \underline{76.6{\scriptsize \textsubscript{±0.3}}} & \underline{81.1{\scriptsize \textsubscript{±0.6}}} & \underline{79.2{\scriptsize \textsubscript{±0.6}}} & \underline{79.2} & 89.0{\scriptsize \textsubscript{±1.0}} & 80.7{\scriptsize \textsubscript{±0.2}} & \underline{80.0{\scriptsize \textsubscript{±1.0}}} & \underline{80.7} \\
 & \textbf{Ours} & \textbf{87.2{\scriptsize \textsubscript{±0.2}}} & \textbf{75.0{\scriptsize \textsubscript{±0.1}}} & \textbf{78.0{\scriptsize \textsubscript{±0.2}}} & \textbf{82.0{\scriptsize \textsubscript{±0.6}}} & \textbf{80.0{\scriptsize \textsubscript{±0.3}}} & \textbf{80.4} & \textbf{89.5{\scriptsize \textsubscript{±0.3}}} & \textbf{81.2{\scriptsize \textsubscript{±0.2}}} & \textbf{81.2{\scriptsize \textsubscript{±0.4}}} & \textbf{81.8} \\
\midrule

\multirow{5}{*}{\rotatebox[origin=c]{90}{\textbf{CPC $\uparrow$}}} 
 & NC & 77.5{\scriptsize \textsubscript{±0.7}} & 31.4{\scriptsize \textsubscript{±39.5}} & 50.8{\scriptsize \textsubscript{±23.2}} & 61.5{\scriptsize \textsubscript{±7.9}} & 42.4{\scriptsize \textsubscript{±5.4}} & 52.7 & - & - & - & - \\
 & \lne & 62.6{\scriptsize \textsubscript{±8.9}} & 71.9{\scriptsize \textsubscript{±12.0}} & 87.0{\scriptsize \textsubscript{±4.5}} & \underline{82.6{\scriptsize \textsubscript{±2.0}}} & \underline{78.5{\scriptsize \textsubscript{±3.6}}} & 76.5 & 69.1{\scriptsize \textsubscript{±8.8}} & 76.7{\scriptsize \textsubscript{±13.8}} & 60.0{\scriptsize \textsubscript{±1.5}} & 73.6 \\
 & \se & 85.0{\scriptsize \textsubscript{±5.1}} & 28.8{\scriptsize \textsubscript{±1.8}} & 75.5{\scriptsize \textsubscript{±0.1}} & 76.7{\scriptsize \textsubscript{±2.4}} & 52.3{\scriptsize \textsubscript{±3.5}} & 63.6 & 73.8{\scriptsize \textsubscript{±7.5}} & 87.4{\scriptsize \textsubscript{±2.5}} & \underline{70.0{\scriptsize \textsubscript{±3.8}}} & 68.7 \\
 & Eigen & \underline{93.9{\scriptsize \textsubscript{±0.6}}} & \underline{85.2{\scriptsize \textsubscript{±8.9}}} & \underline{89.2{\scriptsize \textsubscript{±1.3}}} & 69.5{\scriptsize \textsubscript{±2.8}} & 56.9{\scriptsize \textsubscript{±2.2}} & \underline{78.9} & \underline{80.6{\scriptsize \textsubscript{±5.6}}} & \underline{87.8{\scriptsize \textsubscript{±4.3}}} & 68.0{\scriptsize \textsubscript{±3.0}} & \underline{78.9} \\
 & \textbf{Ours} & \textbf{94.5{\scriptsize \textsubscript{±0.1}}} & \textbf{96.7{\scriptsize \textsubscript{±0.7}}} & \textbf{98.0{\scriptsize \textsubscript{±0.1}}} & \textbf{83.9{\scriptsize \textsubscript{±2.7}}} & \textbf{80.2{\scriptsize \textsubscript{±4.6}}} & \textbf{90.7} & \textbf{90.0{\scriptsize \textsubscript{±1.5}}} & \textbf{97.8{\scriptsize \textsubscript{±0.5}}} & \textbf{70.1{\scriptsize \textsubscript{±1.4}}} & \textbf{88.9} \\
\midrule

\multirow{5}{*}{\rotatebox[origin=c]{90}{\textbf{ECE $\downarrow$}}} 
 & NC & .343{\scriptsize \textsubscript{±.008}} & .484{\scriptsize \textsubscript{±.017}} & .335{\scriptsize \textsubscript{±.013}} & \textbf{.093{\scriptsize \textsubscript{±.007}}} & .157{\scriptsize \textsubscript{±.007}} & .282 & - & - & - & - \\
 & \lne & .055{\scriptsize \textsubscript{±.007}} & \textbf{.044{\scriptsize \textsubscript{±.002}}} & \underline{.079{\scriptsize \textsubscript{±.007}}} & .136{\scriptsize \textsubscript{±.013}} & \underline{.135{\scriptsize \textsubscript{±.010}}} & \underline{.090} & \underline{.143{\scriptsize \textsubscript{±.020}}} & \underline{.072{\scriptsize \textsubscript{±.002}}} & \textbf{.096{\scriptsize \textsubscript{±.005}}} & \underline{.095} \\
 & \se & \underline{.044{\scriptsize \textsubscript{±.003}}} & .193{\scriptsize \textsubscript{±.043}} & .157{\scriptsize \textsubscript{±.015}} & .195{\scriptsize \textsubscript{±.044}} & .379{\scriptsize \textsubscript{±.018}} & .194 & .236{\scriptsize \textsubscript{±.012}} & .292{\scriptsize \textsubscript{±.019}} & .213{\scriptsize \textsubscript{±.014}} & .214 \\
 & Eigen & .056{\scriptsize \textsubscript{±.004}} & .153{\scriptsize \textsubscript{±.010}} & .217{\scriptsize \textsubscript{±.005}} & .325{\scriptsize \textsubscript{±.024}} & .361{\scriptsize \textsubscript{±.025}} & .222 & .171{\scriptsize \textsubscript{±.007}} & .244{\scriptsize \textsubscript{±.009}} & .330{\scriptsize \textsubscript{±.005}} & .232 \\
 & \textbf{Ours} & \textbf{.043{\scriptsize \textsubscript{±.004}}} & \underline{.053{\scriptsize \textsubscript{±.012}}} & \textbf{.042{\scriptsize \textsubscript{±.001}}} & \underline{.109{\scriptsize \textsubscript{±.003}}} & \textbf{.104{\scriptsize \textsubscript{±.012}}} & \textbf{.070} & \textbf{.080{\scriptsize \textsubscript{±.007}}} & \textbf{.041{\scriptsize \textsubscript{±.001}}} & \underline{.127{\scriptsize \textsubscript{±.025}}} & \textbf{.075} \\
\midrule

\multirow{5}{*}{\rotatebox[origin=c]{90}{\textbf{AURAC $\uparrow$}}} 
 & NC & 93.5{\scriptsize \textsubscript{±0.3}} & 67.3{\scriptsize \textsubscript{±0.8}} & 72.4{\scriptsize \textsubscript{±0.4}} & 34.1{\scriptsize \textsubscript{±0.3}} & 48.7{\scriptsize \textsubscript{±2.2}} & 63.2 & - & - & - & - \\
 & \lne & 95.1{\scriptsize \textsubscript{±0.2}} & 79.1{\scriptsize \textsubscript{±0.3}} & 75.3{\scriptsize \textsubscript{±0.7}} & 29.9{\scriptsize \textsubscript{±0.3}} & 51.4{\scriptsize \textsubscript{±0.2}} & 66.2 & \underline{70.7{\scriptsize \textsubscript{±0.1}}} & 67.7{\scriptsize \textsubscript{±0.1}} & 17.8{\scriptsize \textsubscript{±0.7}} & 60.9 \\
 & \se & 95.0{\scriptsize \textsubscript{±0.1}} & 77.5{\scriptsize \textsubscript{±0.3}} & 78.5{\scriptsize \textsubscript{±0.4}} & 37.0{\scriptsize \textsubscript{±0.7}} & 54.2{\scriptsize \textsubscript{±1.0}} & 68.4 & 69.8{\scriptsize \textsubscript{±0.6}} & 67.7{\scriptsize \textsubscript{±0.3}} & 17.7{\scriptsize \textsubscript{±0.2}} & 62.2 \\
 & Eigen & \underline{96.4{\scriptsize \textsubscript{±0.1}}} & \underline{80.3{\scriptsize \textsubscript{±0.2}}} & \underline{80.2{\scriptsize \textsubscript{±0.4}}} & \underline{38.3{\scriptsize \textsubscript{±0.1}}} & \underline{57.5{\scriptsize \textsubscript{±1.3}}} & \underline{70.6} & 70.3{\scriptsize \textsubscript{±0.2}} & \underline{70.3{\scriptsize \textsubscript{±0.1}}} & \underline{20.3{\scriptsize \textsubscript{±0.6}}} & \underline{64.2} \\
 & \textbf{Ours} & \textbf{96.7{\scriptsize \textsubscript{±0.1}}} & \textbf{81.0{\scriptsize \textsubscript{±0.3}}} & \textbf{80.9{\scriptsize \textsubscript{±0.5}}} & \textbf{38.7{\scriptsize \textsubscript{±0.1}}} & \textbf{59.4{\scriptsize \textsubscript{±0.8}}} & \textbf{71.3} & \textbf{73.8{\scriptsize \textsubscript{±0.2}}} & \textbf{71.9{\scriptsize \textsubscript{±0.2}}} & \textbf{20.7{\scriptsize \textsubscript{±0.2}}} & \textbf{65.4} \\
\bottomrule
\end{tabular}
}
\end{table}

\subsection{\ref{r:classify}: Discrimination}
\label{sec:auroc}
We first evaluate metrics on \ref{r:classify}, i.e., whether the metrics can distinguish task instances that the MLLM will get correct ($\task_c$) or wrong ($\task_w$), measured via AUROC.
\Cref{tab:r12_results_rotate} shows that \algname \emph{consistently achieves the best performance} with an average AUROC of 
$0.80$
on image-text datasets, excelling
also in challenging datasets like OKVQA and AdVQA where multimodal-specific methods like \cu struggle due to adversarial and out-of-distribution scenarios.
Beyond image-text input tasks, \algname also \emph{demonstrates robust discrimination performance across audio-text, video-text tasks}, underscoring its modality generalizability \ref{r:generalizability}.
In \Cref{app:tpr} \Cref{tab:tpr_results}, we also show how \algname framework's better AUROC performance for \ref{r:classify} translates to consistently higher True Positive Rates (TPR) given FPR requirements.

\subsection{\textbf{\ref{r:r2a}}, \textbf{\ref{r:r2b}}: Risk-score quality}
\label{sec:calibration}
Similar to past works \citep{guoCalibrationModernNeural2017b}, we sort task instances $\task$ by uncertainty metric $\metric(\model,\task)$ and put them into equally-sized bins $b_j$ (results are robust to bin size, see \cref{app:correctness}). Each bin is associated with its highest metric value $\metric_j$, and the average error $\hat r_j := \frac{1}{|b_j|}\sum_{t\in b_j}\mathbb{I}\{a(M,t)=0\}$.

\textbf{Calibration Pearson Correlation (CPC)(\textbf{\ref{r:r2a}}).} 
We define CPC score as the Pearson correlation between $\metric_j$ and $\hat r_j$ across bins (higher is better). This quantifies how close the reliability curve is to an affine trend, satisfying \ref{r:r2a}.
\Cref{tab:r12_results_rotate} shows that \algname \emph{consistently performs better than baselines across most settings}, achieving an average CPC of $\sim 0.90$ across all modality tasks (\ref{r:generalizability}), more than $11\%$ higher than the next best metric. Note that \algname also produces \emph{more stable and reliable results with consistently high CPC}, unlike other baselines with performance that fluctuates greatly depending on the specific task.

\textbf{Expected Calibration Error (ECE) (\textbf{\ref{r:r2b}}).} The strong linearity indicated by CPC score suggests that a simple scaling process could allow \algname to empirically track approximate model error well and satisfy \textbf{\ref{r:r2b}}. 
We evaluate the ECE \citep{guoCalibrationModernNeural2017b} of metrics by using an \emph{unlabeled} development set of instances (5\% of dataset) to compute $\tilde{\metric}$ via min-max scaling before computing the ECE. Intuitively, development sets that contain both very hard ($\tilde{\metric}\approx1$) and easy ($\tilde{\metric}\approx0$) instances could help scale a strongly linear metric to satisfy \textbf{\ref{r:r2b}}.
\algname{} achieves a very low ECE on almost all datasets, with an average of $0.075$ across all modalities (\Cref{tab:r12_results_rotate}), while \eigen{} and \se{} have much higher normalized-score ECE. 
Note that $\tilde{\metric}$ is an unlabeled error probability proxy, and may be sensitive to the development set's outliers. With a labeled development set (e.g., 20\% of instances), supervised post-hoc Platt scaling further reduces \algname's average ECE from $0.070$ to $0.029$ on image-text datasets, which remains the lowest among all methods (vs. $0.036$ for the best baseline, and lower than $\unadjmethod$ ($0.050$) or $\quadent$ ($0.038$) alone, \cref{app:platt}).

\begin{figure}[ht]
    \begin{center}
\includegraphics[width=\linewidth]{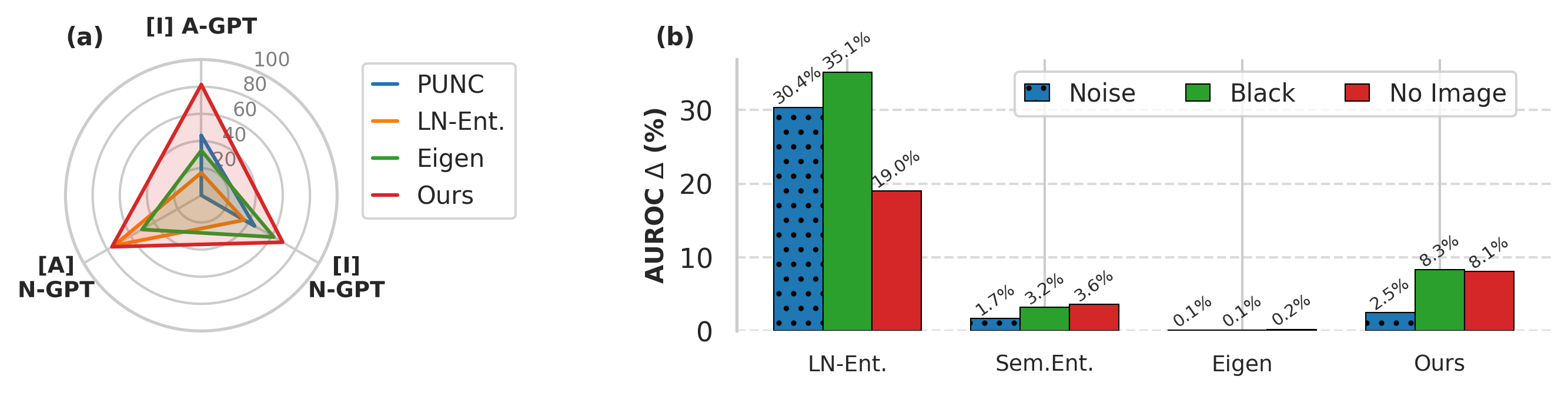}
    \end{center}
    \caption{\textbf{(a)} Pearson Correlation between uncertainty metrics and CLIP/CLAP score for image [I] and audio [A] generation tasks (\ref{r:generalizability}'), using AnyGPT (A-GPT) and NExT-GPT (N-GPT); \algname consistently achieves the highest correlation. \textbf{(b)} Multimodal coherence \ref{r:coherence}: Decrease (denoted as $\Delta$) in AUROC when image-input information is (1) corrupted with noise, (2) replaced with a black image, or (3) removed.}
\label{fig:r3_r4}
\end{figure}

\subsection{\ref{r:generalizability}, \ref{r:coherence}, \ref{r:efficiency}: Design guidelines}
\label{sec:design_desiderata}

\textbf{Multimodal generalizability (\ref{r:generalizability}).} As seen in \cref{tab:r12_results_rotate}, \algname, without modality-specific modifications or external tools, \emph{consistently performs well in effectiveness} (\ref{r:classify}-\textbf{\ref{r:r2b}}) across multiple input modality tasks (image-text, audio-text, video-text) and hence satisfy \ref{r:generalizability}, empirically supporting our approach in using MLLMs' inherent multimodality capabilities to achieve \ref{r:generalizability}. In contrast, metrics that rely on modality-specific input, such as \cu, cannot be directly applied to other input modalities to satisfy \ref{r:generalizability}. The other baselines (\lne, \se, \eigen) do not explicitly consider multimodal input and could be adapted to the various tasks, but consistently perform worse than \algname.

\textbf{Non-text output generation tasks (\ref{r:generalizability}').} 
To test \ref{r:generalizability}', we also 
ran experiments on image--MS-COCO caption \citep{capeval2015}, and audio--AudioCaps \citep{kim-etal-2019-audiocaps} generation using any-to-any MLLMs NExT-GPT \citep{wu24next} and AnyGPT \citep{zhan-etal-2024-anygpt}.
We evaluate the uncertainty metrics on \ref{r:r2a} based on the Pearson correlation between the metrics and the negative of MLLM responses' continuous quality scores (image: CLIP score \citep{hessel-etal-2021-clipscore}; audio: CLAP score \citep{elizalde2023clap}), similar to how CPC is computed for metrics and estimated instance error in \cref{sec:calibration}.
\se requires text-specific processing with external tools and cannot be applied. For these settings, we also compare with PUNC \citep{franchi2025towards}, an image generation-specific uncertainty metric. \Cref{fig:r3_r4} (a) shows that \emph{\algname has consistently strong correlation with image and audio quality, outperforms all baselines} across different modalities, and hence satisfies \ref{r:generalizability}'. This allows \algname to assess whether a task instance might be challenging for MLLMs to produce high quality image/audio responses to, which to our knowledge has not been well-studied.

\textbf{Multimodal coherence (\ref{r:coherence}).} To further demonstrate that \algname considers multimodality information despite not using modality-specific external tools or methods, we analyze post-generation (i.e. holding MLLM output fixed) whether metric performance degrades when image-input information is (1) corrupted with noise, (2) replaced with a black image, or (3) removed. A metric that satisfies \ref{r:coherence} should show performance degradation for (1), which worsens for (2) and (3) since useful signals would be removed. 
\Cref{fig:r3_r4} (b) \emph{shows that \algname exhibits this behavior well and satisfies \ref{r:coherence}}, along with \se and \lne.
In contrast, 
\eigen is nearly invariant to image removal, indicating that it measures response dispersion without considering multimodal coherence.

\textbf{Computational efficiency (\ref{r:efficiency}).} Both \cu and \se perform poorly in \ref{r:efficiency} as they require expensive computation (e.g., pairwise response evaluation), external tools and additional training (for \cu). Compared to \algname and other baselines, they take up to $1000\times$ more compute overhead time beyond sampling (i.e., beyond MLLM sampling cost), taking ${\sim} 9$s/sample v.s. ${\sim} 8\mathrm{e}{-4}$s/sample) as shown in \Cref{fig:run_time_num_gen}(a), \cref{app:efficiency}. 
In general, \algname remains practical at small sampling budgets, requiring only a budget of $k=5$ for a favorable performance-cost trade-off (see \cref{fig:run_time_num_gen}(b), \cref{app:efficiency}). \algname also provides major performance boosts over single-sample methods while incurring only small additional cost. On VQAv2, with only $k=2$, it raises AUROC from $0.63$ for the best single-pass baseline to $0.85$ (a $34\%$ relative gain) while adding only about $0.3$s per instance, and reaches $0.87$ at $k=5$ (similar consistent gains for other datasets can be seen in \cref{app:efficiency}, \cref{tab:cost_perdataset} and \cref{tab:cost_k}). While there will be some practical settings that require the lowest possible inference time and single-sample methods would be needed, we believe that \algname is complementary to them and provide advantages on a different part of the quality-efficiency trade-off.

\subsection{Practical applications}
\label{sec:effective_practical}

\textbf{Selective answering.} \label{sec:aurac}
We consider a practical scenario where a user has a small local MLLM for a task, and a limited budget to route/escalate the most uncertain task instances to a more capable MLLM or human expert. We evaluate the performance of uncertainty metrics in this setting via the Area Under the Rejection-Accuracy Curve (AURAC) \citep{hullermeierAleatoricEpistemicUncertainty2021}, 
which measures the accuracy gain achieved when instances are progressively rejected based on their uncertainty rankings.
As can be seen in \Cref{tab:r12_results_rotate}, \algname \emph{consistently achieves the highest AURAC across all datasets.} 
In \Cref{app:statistical-t-test}, we also present results of statistical tests across various MLLM model-dataset combinations, demonstrating that \algname's performance gains over baselines are statistically significant.

\begin{figure*}[t]
    \centering
    \includegraphics[width=\linewidth]{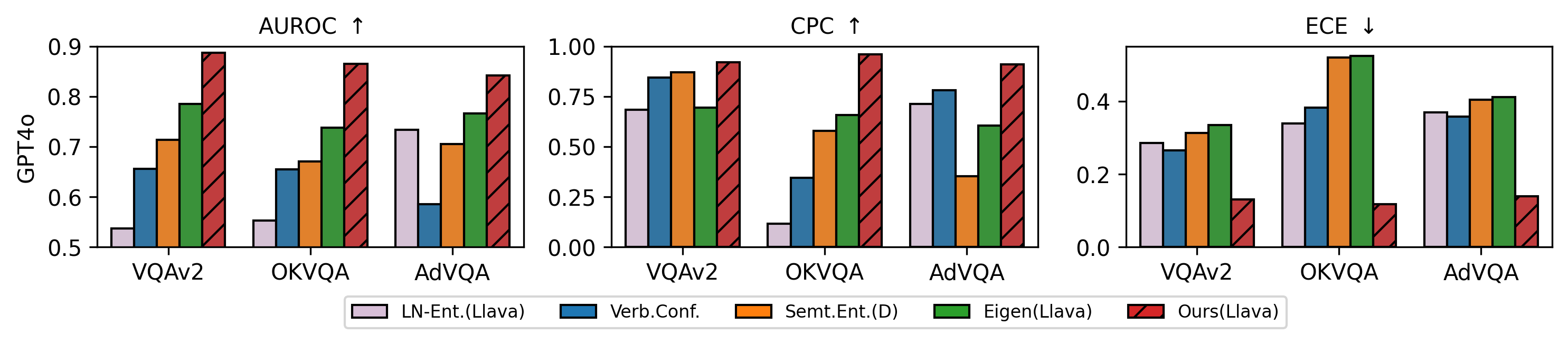}
    \caption{Performance of uncertainty metrics in blackbox settings across image-text QA datasets. \se (D) indicates its discrete version, (Llava) indicates Llava as the white-box proxy model.
    }
    \label{fig:blackbox}
\end{figure*}

\textbf{Blackbox Models.}
\label{sec:blackbox}
While blackbox MLLMs do not provide internal embeddings, we could still apply \algname by using a much smaller whitebox proxy MLLM to process the blackbox model's responses to generate embeddings and probabilities for computing our metric. While this incurs inference cost of the small proxy model and $\quadent_\task$ becomes a cross-model incoherence estimate rather than exact quadratic entropy, \algname remains effective as it is robust to perturbations in $(\phi_i,p_i)$ and capable MLLMs share aligned multimodal representations, consistent with the Platonic representation hypothesis \citep{platonic}. As shown in \Cref{fig:blackbox}, \algname{} consistently and significantly outperforms baselines, including Verbalized Confidence (Verb.Conf) \citep{xiong2024llmsexpressuncertaintyempirical} which runs directly on API models, when assessing GPT4o's \citep{hurst2024gpt} uncertainty on VQAv2, OKVQA, and AdVQA with \llava as the whitebox proxy. Additional experiments using other blackbox and proxy models show similarly strong results (\Cref{app:blackbox}). To better characterize when such a white-box proxy approach is expected to work well for \algname, we analyzed how using proxies differing in size, backbone, image encoding, and model family (\Cref{app:proxy_mismatch}) impact the proxy's estimate for the target model's (\llava) \algname score. Whitebox proxy estimates for discrimination \ref{r:classify} are robust across proxy model variations (AUROC of $0.840$-$0.875$ vs. $0.876$ when the actual model is used rather than a proxy), as the semantic volume $\unadjmethod$ is stable across proxies. However, the whitebox proxy estimates for risk-score quality \ref{r:risk} are more sensitive. When using a proxy, $\quadent$ also reflects likelihood mismatch between the proxy and target models, leading to degradation of unlabeled ECE if there is large model mismatch, though using a small labeled calibration set largely recovers it. In practice, we recommend proxy-based \algname{} primarily for assessing discrimination \ref{r:classify}.

\section{Related Works}
\label{sec:related_works}
\textbf{Modality-specific methods.} While MLLMs' hallucination and miscalibration problems are well known \citep{chenUnveilingUncertaintyDeep2025, chairrohrbach2018object, baiHallucinationMultimodalLarge2024}, task instance-specific uncertainty quantification for MLLMs remains relatively underdeveloped. Most works are focused on the `image-text input, text output' modality, with image-specific approaches and external tools, violating \ref{r:generalizability}. This includes works relying on the use of external reference/entailment models \citep{zhangVLUncertaintyDetectingHallucination2024, mhalsun2023aligning, gaiveliu2023aligning}, supervised training of classifiers \citep{liReferencefreeHallucinationDetection2024}, or large numbers of modality-specific query perturbations \citep{khan2024consistency,zhangVLUncertaintyDetectingHallucination2024} to test model consistency (also violating \ref{r:efficiency}). Even with additional compute or access to external tools, these methods also tend to underperform \algname in (\ref{r:classify}-\textbf{\ref{r:r2b}}), as shown in \cref{sec:exp}.

\textbf{LLM uncertainty methods.} While designed only for text-input settings, we found that some LLM uncertainty metrics could be adapted for multimodal input and potentially achieve better effectiveness (e.g., \ref{r:classify} on discrimination) compared to modality-specific methods (see \cref{sec:exp}). This includes methods based on lexical (token probability) distributions \citep{malinin2021uncertainty}, semantic clusters/graphs derived from external text entailment models \citep{kuhnSemanticUncertaintyLinguistic2023, nikitinKernelLanguageEntropy2024,lin2024generating}, semantic embeddings of sampled responses \citep{chenLLMsInternalStates2024, qiu2024semantic}, and prompting \citep{xiong2024llmsexpressuncertaintyempirical,tian2023just}.
However, these approaches still typically violate several desiderata (e.g., \ref{r:coherence} by not considering response incoherence with multimodal input) and consistently underperforms \algname. Compared to other forms of response probability-based uncertainty metrics \citep{malinin2021uncertainty}, our incoherence score/quadratic entropy $\quadent$ approach is novel and yields consistent performance gains over baselines, which widens further when augmented with $\unadjmethod$.

Many MLLM and LLM uncertainty works rely on computing discrete entropy measures \citep{malinin2021uncertainty,nikitinKernelLanguageEntropy2024,zhangVLUncertaintyDetectingHallucination2024}. However, it is unclear how to compare entropy values across different support sets (e.g., distributions defined on 2 versus 5 classes), especially when the support set is determined by external models, making them potentially hard to use in practice. 
\eigen \citep{chenLLMsInternalStates2024} considers differential entropy in the sentence embedding space under a Gaussian assumption, resulting in a log determinant form that bears some similarity to our unadjusted semantic volume term $\unadjmethod_\task$. In contrast, \algname has key advantages including (i) incorporates probability-based incoherence scores through a DPP-inspired quality-diversity kernel rather than assuming that all responses have equal quality, which also provides important multimodal coherence signals for \ref{r:coherence}, and (ii) provides interpretation of semantic volume and quadratic entropy that naturally motivates the form of the metric without requiring strong distributional assumptions, among others (\cref{app:eigen}). Empirically, we see \algname consistently outperforming \eigen across tasks (\Cref{tab:r12_results_rotate}).

\section{Conclusion}
\label{sec:conclusion}
We propose \algname, an inference-time framework for MLLM uncertainty without the need for modality-specific tools.
By weighting semantic volume with probability-based incoherence scores, \algname considers both semantic diversity and quadratic entropy of MLLM responses. Empirically, \algname outperforms baselines across image, audio and video QA benchmarks, and uncertainty estimation for blackbox models via whitebox proxies. 
We note several limitations. First, like other training-free uncertainty metrics without external ground truth, \algname targets confabulations and relies on response variation. It cannot detect systematic errors where an MLLM consistently produces the same high-probability wrong answer (yielding both low $\unadjmethod$ and low $\quadent$, e.g., due to flawed training data or data mixtures \citep{mazumder2023dataperf, xu2024datacentric, chen2026duet}). Correspondingly, our concentration result (\cref{prop:app:concentration}) depends on a positive expected-score gap $\Delta$ between incorrect and correct instances, which must be verified empirically. Second, although \algname{} is more efficient than other multi-sample baselines (e.g., \se) and effective at small $k$ (\cref{app:efficiency}), it still incurs higher latency than single-sample heuristics, especially for long output. Similar to other metrics, \algname{}'s performance also depends on practical choices such as sampling temperature, length normalization for long-form generation, and proxy-target model alignment in blackbox settings (\cref{app:ablation_study,app:blackbox}). Future work will explore maintaining performance with fewer generations and extending \algname{} to downstream applications requiring uncertainty quantification, such as active learning and experimental design \citep{settles2009active, kirsch2019batchbald, pinnacle2024, pied2025}.

\section*{Acknowledgments}
We thank Pang Wei Koh for the helpful discussions. This research/project is supported by the National Research Foundation, Singapore under its AI Singapore Programme (AISG Award No: AISG2-PhD/2023-01-039J) and is part of the programme DesCartes which is supported by the National Research Foundation, Prime Minister’s Office, Singapore under its Campus for Research Excellence and Technological Enterprise (CREATE) programme. This research is supported by the National Research Foundation Singapore and the Singapore Ministry of Digital Development and Innovation, National AI Group under the AI Visiting Professorship Programme (award number AIVP-2024-001).

\newpage

\small
\bibliography{umpire}
\bibliographystyle{unsrtnat}

\newpage
\appendix
\onecolumn

\normalsize

\section{Theoretical analysis and intuition}
\label{app:theory_result}

In this section, we provide some theoretical analysis on the \algname metric under simplifying assumptions, to help build intuition for why the metric has strong empirical results that we presented in \cref{sec:exp}.

\subsection{Notation and problem setup}
\label{app:theory_setup}

We first recap the problem setup. For a given task instance $\task \in \taskspace$, let $Y\sim \mathbb{P}_\model(\cdot\mid q_t)$ be i.i.d sampled responses from the model's conditional probability\footnote{For notational simplicity, we drop the subscript $\task$ from $Y_\task$ which depends on the task instance $t$. Note that the i.i.d. condition applies even with standard MLLM inference practices such as nucleus sampling. While nucleus sampling truncates full model the distribution, the resulting response samples are still strictly i.i.d draws from the truncated distribution. In fact, nucleus sampling also improves the concentration results in \cref{prop:app:concentration} that we discuss later, which depends on the bounded difference constant, as it effectively removes the long tail of very low probability tokens (which in practice are also often noise).}. Let $\phi_i \in \mathbb{R}^d$ denote the normalized (we enforce $\|\phi_i\|_2 =1$) embedding (row vector) for sampled response $y_i$, and $\Phi \in \mathbb{R}^{k \times d}$ with rows $\phi_1^\top, \dots, \phi_k^\top$ be the semantic embedding matrix. Let $p_i \in (0, 1]$ be the model-generated probability for sampled response $y_i$. We define the incoherence scaling matrix be $C = \mathrm{diag}(c_1, \dots, c_k)$ with $c_i := \exp(\alpha(1 - p_i))$, for scalar $\alpha \geq 0$. Our proposed \algname score (\cref{eq:umpire}), inspired by quality-diversity kernels used in DPP, is
$$
\method_\task =  \frac{1}{2k}  \log \det \left [ C_\task \left ((\Phi_\task\Phi_\task^\top+ \epsilon I_k) \right ) C_\task \right ],
$$
where $\epsilon >0$ is a jitter term.

\paragraph{\algname metric decomposition.} We first provide a basic derivation of how the incoherence-adjusted semantic volume presented as \cref{eq:umpire} can be expressed as \cref{eq:casev} in \cref{sec:method}.

We define the unadjusted semantic volume term
\begin{equation}
\unadjmethod_\task \coloneqq \frac{1}{2k}\log \det(\Phi_\task\Phi_\task^\top + \epsilon I_k),
\label{eq:app:V_tilde}
\end{equation}
and the empirical incoherence average
\begin{equation}
\quadent_\task \coloneqq \frac{1}{k}\sum_{i=1}^k (1-p_i)\in [0,1].
\label{eq:app:qbar}
\end{equation}

\begin{proposition}[\algname decomposition]
\label{prop:app:decomp}
For $\method_\task$ defined in \cref{eq:umpire}, we have
\begin{equation}
\method_\task = \unadjmethod_\task + \alpha \quadent_\task
\label{eq:app:decomp}
\end{equation}
\end{proposition}
\begin{proof}
As $C$ and $(\Phi_\task \Phi_\task^\top + \epsilon I_k)$ are $k \times k$ square matrices, we have
$$
\det(C (\Phi_\task \Phi_\task^\top + \epsilon I_k) C) = \det(C) \det((\Phi_\task \Phi_\task^\top + \epsilon I_k)) \det(C) =  \det((\Phi_\task \Phi_\task^\top + \epsilon I_k))\det(C)^2,
$$
hence
$$
\log\det(C (\Phi_\task \Phi_\task^\top + \epsilon I_k) C) = \log\det((\Phi_\task \Phi_\task^\top + \epsilon I_k)) + 2\log\det(C).
$$
Substituting $C$ which is a diagonal matrix with $c_i = \exp(\alpha(1 - p_i))$, we have
$$
2\log\det(C) = 2\sum_{i=1}^{k} \log c_i = 2\alpha \sum_{i=1}^{k} (1 - p_i).
$$
We can obtain the final claimed decomposition by dividing throughout by $2k$.
\end{proof}

\paragraph{Re-expression of $\widetilde{V}_\task$ with empirical second moments.}
Next, we re-express $\widetilde{V}_\task$ in terms of an empirical second moment term that enables more convenient analysis. 

\begin{lemma}[Second moment form of semantic volume]
\label{lem:app:detlemma}
Let $S_k := \frac{1}{k}\Phi_\task^\top\Phi_\task \in \mathbb{R}^{d\times d}$. Then
\begin{equation}
\det(\Phi_\task\Phi_\task^\top + \epsilon I_k)
= \epsilon^k \det\!\Big(I_d + \frac{k}{\epsilon} S_k\Big),
\label{eq:app:detlemma}
\end{equation}
and therefore
\begin{equation}
\unadjmethod_\task
= \frac{1}{2}\log \epsilon
+ \frac{1}{2k}\log\det\!\Big(I_d + \frac{k}{\epsilon} S_k\Big),
\label{eq:app:VvolSk}
\end{equation}
where $I_k$ and $I_d$ are $k \times k$ and $d \times d$ identity matrices.
\end{lemma}
\begin{proof}
Apply the matrix determinant lemma $\det(I_k + AB)=\det(I_d + BA)$ with
$A=\epsilon^{-1/2}\Phi_\task$ and $B=\epsilon^{-1/2}\Phi_\task^\top$:
\[
\det(\Phi_\task\Phi_\task^\top+\epsilon I_k)
=\epsilon^k \det\!\Big(I_k + \frac{1}{\epsilon}\Phi_\task\Phi_\task^\top\Big)
=\epsilon^k \det\!\Big(I_d + \frac{1}{\epsilon}\Phi_\task^\top\Phi_\task\Big).
\]
Substitute $\Phi_\task^\top\Phi_\task = kS_k$ to obtain~\eqref{eq:app:detlemma} and~\eqref{eq:app:VvolSk}.
\end{proof}

\subsection{Quadratic-entropy interpretation of the incoherence term}
\label{app:theory:quadratic}

We start by discussing how $\quadent_\task$ can be interpreted. 

\begin{definition}[Quadratic entropy]
\label{def:app:H2}
Let $\mathbb{P}_\model(\cdot\mid \query_t)$ denote the model distribution over discrete responses $Y$ for a fixed instance $t$. We define the quadratic entropy 

\begin{equation}
H_2(Y) \coloneqq 1 - \sum_y \mathbb{P}_\model(y\mid \query_t)^2.
\label{eq:app:H2}
\end{equation}
\end{definition}

Note that if $Y, Y' \stackrel{iid}{\sim} \mathbb{P}(\cdot \mid \query_t)$, $\mathbb{P}(Y=y') = \sum_y \mathbb{P}_\model(y\mid \query_t)^2$. Hence, $H_2(Y)$ can also be interpreted as the probability that two sampled responses from $\mathbb{P}_\model$ do not match.

\begin{lemma}[Incoherence term as Monte Carlo estimate of $H_2$]
\label{lem:app:QasH2} Let $Y \stackrel{iid}{\sim} \mathbb{P}(\cdot \mid \query_t)$ and $p_i = \mathbb{P}(y_i\mid \query_t)$ for each sampled $y_i$. The incoherence term $\quadent_\task=\frac{1}{k}\sum_i (1-p_i)$ is an unbiased Monte Carlo estimator of $H_2(Y)$.
\end{lemma}

\begin{proof}
Note that 
\[
\mathbb{E}\big[1-\mathbb{P}(Y\mid \query_t)\big]
=\sum_y \mathbb{P}(y\mid \query_t)\big(1-\mathbb{P}(y\mid \query_t)\big)
=1-\sum_y \mathbb{P}(y\mid \query_t)^2
=H_2(Y).
\]
Since $Y\stackrel{iid}{\sim} \mathbb{P}(\cdot \mid \query_t)$ and $p_i = \mathbb{P}(y_i\mid \query_t)$ for each sampled $y_i$, $\quadent_\task=\frac{1}{k}\sum_i (1-p_i)$ is an unbiased empirical Monte Carlo estimate of $\mathbb{E}\big[1-\mathbb{P}(Y\mid \query_t)\big]$ and hence of $H_2(Y)$.
\end{proof}

\subsection{Small $Q$ rules out multiple high-probability semantic modes}\label{app:smallQ-modes}

Intuitively, as described in \cref{sec:prac_consi}, it is clear that small $\quadent_\task$ can help rule out the presence of multiple high probability semantic modes in the MLLM response distribution for a given task instance $\task$. Here, we make things more explicit and formalize this with a few basic derivations. 

Let $Y\sim P(\cdot\mid \query_t)$ be the sampled MLLM response, and let $Z=g(Y)\in\{1,\ldots,m\}$ be a mapping that assigns each
response to a semantic mode or cluster. We denote the probability mass associated with each mode as $w_j:=\mathbb{P}(Z=j)$, and correspondingly the quadratic entropy over semantic modes as 

\[
H_2(Z) \coloneqq 1-\sum_{j=1}^m w_j^2 = \mathbb{P}(Z\neq Z'),
\]
where $Z,Z'$ are i.i.d. samples.

\begin{lemma}[Coarsening decreases quadratic entropy]
\label{lem:app:coarsen}
For any mapping $Z=g(Y)$, we have $H_2(Z)\le H_2(Y)$.
\end{lemma}
\begin{proof}
Let $\mathcal{C}_j=\{y:g(y)=j\}$ so that $w_j=\sum_{y\in\mathcal{C}_j} \mathbb{P}_\model(y\mid \query_t)$.
Since the probabilities are non-negative, $(\sum_{y\in\mathcal{C}_j} \mathbb{P}_\model(y\mid \query_t))^2 = \sum_{y\in\mathcal{C}_j} \mathbb{P}_\model(y\mid \query_t)^2 + \sum_{y\in\mathcal{C}_j,k\neq l} \mathbb{P}_\model(y_k\mid \query_t)\mathbb{P}_\model(y_l\mid \query_t) \ge \sum_{y\in\mathcal{C}_j} \mathbb{P}_\model(y\mid \query_t)^2$.
Summing over all semantic modes $j$ yields $\sum_j w_j^2 \ge \sum_y P(y\mid \query_t)^2$, and hence
$H_2(Z) = 1-\sum_j w_j^2 \le 1-\sum_y P(y\mid \query_t)^2 = H_2(Y)$.
\end{proof}

Hence, even though $\quadent_\task$, which estimates $H_2(Y_\task)$, does not capture any notion of semantic distance among sampled responses, it still provides an upper bound on the dispersion of probability mass among semantic modes/clusters (if any). Small values of $\quadent_\task$ will thus provide some information regarding the number of semantic modes that exist in the MLLM's response distribution for a task instance $\task$. We can analyze this with two corollaries.

\begin{corollary}[Dominant-mode lower bound]
\label{cor:app:dominant}
Let $w_{\max}:=\max_j w_j$. Then
\[
w_{\max} \;\ge\; \sum_{j=1}^m w_j^2
\;=\;
1-H_2(Z)
\;\ge\;
1-H_2(Y)
\;\approx\;
1-Q.
\]
\end{corollary}

Therefore, small $Q$ implies that most probability mass would be concentrated in a single semantic mode.

\begin{corollary}[Ruling out many high-mass modes]
\label{cor:app:manymodes}
Fix $\beta\in(0,1)$ and integer $r\ge 2$ with $r\beta\le 1$.
If at least $r$ modes satisfy $w_j\ge \beta$, then necessarily
\begin{equation}
\quadent \approx H_2(Y) \geq H_2(Z) \ge
1 - \Big(1-(r-1)\beta\Big)^2 - (r-1)\beta^2.
\label{eq:app:manymodes}
\end{equation}
\end{corollary}

\begin{proof}
To minimize $H_2(Z)=1-\sum_j w_j^2$, we can maximize $\sum_j w_j^2$. Under the constraint that $r$ modes have mass at least $\beta$, this can be done by setting $w_j=\beta$ for $r-1$ modes and concentrating the rest of the probability mass onto one additional mode:
$w_1=1-(r-1)\beta$ and $w_2=\cdots=w_r=\beta$ (with the others being $0$).
This yields $\sum_j w_j^2 = (1-(r-1)\beta)^2+(r-1)\beta^2$, implying~\eqref{eq:app:manymodes}.
\end{proof}

Hence, observing small $\quadent_\task$ precludes the existence of many semantic modes that each carry non-trivial probability mass, and indicate that most probability mass would likely be concentrated in a single semantic mode. Consequently, a regime where we observe large semantic volume $\unadjmethod_\task$ but small $\quadent_\task$ will be consistent with a \emph{dominant mode with a small number of outliers},
rather than multiple high-probability semantic modes.

\subsection{Semantic volume increases with between-mode spread (mixture intuition)}
\label{app:theory:mixture}

We now relate the volume term to the spread of semantic modes under a standard mixture-of-embeddings approximation.

\begin{assumption}[Mixture-of-modes embedding model (optional)]
\label{ass:app:mixture}
The embedding $\phi$ of a sampled response for fixed instance $t$ follows a finite mixture:
$\phi_\task \sim \sum_{j=1}^m w_j \mathcal{D}_j$ with mixture weights $w_j$, means $\mu_j$, and within-mode covariances $\Sigma_j$.
Let $\bar{\mu}=\sum_j w_j\mu_j$, define
\[
\Sigma_{\mathrm{within}} := \sum_j w_j \Sigma_j,
\qquad
\Sigma_{\mathrm{between}} := \sum_j w_j (\mu_j-\bar{\mu})(\mu_j-\bar{\mu})^\top,
\qquad
\Sigma_{\mathrm{mix}} := \Sigma_{\mathrm{within}}+\Sigma_{\mathrm{between}}.
\]
\end{assumption}

Because embeddings may live near a low-dimensional manifold, $\Sigma_{\mathrm{mix}}$ may be singular.
We therefore reason about the regularized determinant $\det(\Sigma_{\mathrm{mix}}+\epsilon I)$, consistent with
the jitter used in our \algname framework.

\begin{proposition}[Monotonicity under increased between-mode spread]
\label{prop:app:psdmono}
Fix $\epsilon>0$. If $\Sigma'_{\mathrm{between}} \succeq \Sigma_{\mathrm{between}}$ (PSD order) and
$\Sigma'_{\mathrm{mix}} := \Sigma_{\mathrm{within}}+\Sigma'_{\mathrm{between}}$, then
\begin{equation}
\det(\Sigma'_{\mathrm{mix}}+\epsilon I) \;\ge\; \det(\Sigma_{\mathrm{mix}}+\epsilon I).
\label{eq:app:detmono}
\end{equation}
Consequently, regimes that shift probability mass toward more separated semantic modes (larger between-mode spread)
lead to a larger population-level regularized determinant, which is the object approximated by the empirical volume term.
\end{proposition}
\begin{proof}
Let $D := \Sigma'_{\mathrm{between}}-\Sigma_{\mathrm{between}}\succeq 0$ so
$\Sigma'_{\mathrm{mix}}+\epsilon I = (\Sigma_{\mathrm{mix}}+\epsilon I)+D$.
By Weyl's monotonicity, each eigenvalue of $(\Sigma_{\mathrm{mix}}+\epsilon I)+D$ is at least the corresponding eigenvalue
of $\Sigma_{\mathrm{mix}}+\epsilon I$. Since determinants are products of eigenvalues,~\eqref{eq:app:detmono} follows.
\end{proof}

Consequently, for large $k$ where sample covariance $S_k = \frac{1}{k}\Phi_\task^T\Phi_\task \approx \Sigma_{\mathrm{mix}}$, $\det(S_k+ \epsilon I)$, or the unadjusted semantic volume term in our proposed metric, increases with $\Sigma_{\mathrm{between}}$. When model uncertainty results in confabulations, where the model start to generate more spread out semantic modes in its responses, $\Sigma_{\mathrm{between}}$ increases, resulting in an increase in our \algname metric.

\subsection{Concentration and ranking consistency with finite $k$}
\label{app:theory:concentration}

Finally, we show that $\method_\task$ concentrates around its conditional expectation
(for a given task instance $\task$) as $k$ grows, and this helps with ranking consistency, supporting the strong empirical performance of \algname in meeting desiderata \ref{r:classify}. We show this via bounded differences and and relies only on (i) $\|\phi_i\|_2=1$, (ii) $p_i\in(0,1]$, and (iii) $\epsilon>0$ as we outlined in \cref{app:theory_setup}.

\begin{proposition}[Concentration and ranking consistency]
\label{prop:app:concentration}
Fix $\epsilon>0$ and $\alpha\ge 0$. Conditioned on a task instance $\task$,
let $(\phi_i,p_i)_{i=1}^k$ be i.i.d. as in Appendix~\ref{app:theory_setup}.
Define
\[
\method_\task
=\unadjmethod_\task + \alpha\quadent_\task,
\qquad
\quadent_\task=\frac{1}{k}\sum_{i=1}^k(1-p_i),
\qquad
\unadjmethod_\task=\frac{1}{2k}\log\det(\Phi_\task\Phi_\task^\top+\epsilon I_k).
\]
Then for any $\eta>0$,
\begin{equation}
\Pr\!\left(\big|\method_\task-\mathbb{E}[\method_\task]\big|>\eta\right)
\le
2\exp\!\left(-\frac{2k\eta^2}{L^2}\right),
\label{eq:app:mcdiarmid}
\end{equation}
where one may take
\[
L := \alpha + \tfrac{1}{2}\log\!\Big(\frac{1+\epsilon}{\epsilon}\Big)
= \alpha + \tfrac{1}{2}\log\!\Big(1+\frac{1}{\epsilon}\Big).
\]
Moreover, for two instances $\task_a,\task_b$ with a gap
$\mathbb{E}[\method_{\task_a}]-\mathbb{E}[\method_{\task_b}]\ge \Delta>0$,
\begin{equation}
\Pr\!\left(\method_{\task_a}\le \method_{\task_b}\right)
\le
4\exp\!\left(-\frac{k\Delta^2}{2L^2}\right).
\label{eq:app:ranking_mcdiarmid}
\end{equation}
\end{proposition}

\begin{proof}
We first show bounded differences for the incoherence term $\quadent_\task$ and unadjusted semantic volume terms $\unadjmethod_\task$ respectively.

For the incoherence term $\quadent_\task$, if we replace one sample $(\phi_i,p_i)$ by $(\phi_i',p_i')$, then
\[
\left|\quadent_\task-\quadent_\task'\right|
=
\left|\frac{1}{k}\big((1-p_i)-(1-p_i')\big)\right|
\le \frac{1}{k},
\]
since $1-p\in[0,1]$. Hence
\[
\left|\alpha\quadent_\task-\alpha\quadent_\task'\right|\le \frac{\alpha}{k}.
\]

For the unadjusted semantic volume term $\unadjmethod_\task$, let $A := \Phi_\task\Phi_\task^\top+\epsilon I_k$ and $A':=\Phi_\task'\Phi_\task'^\top+\epsilon I_k$,
where $\Phi_\task'$ differs from $\Phi_\task$ only in row $i$.
Permute indices so that $i=1$ (determinants are permutation-invariant) and write
\[
A=\begin{pmatrix}
a_{11} & a_{1,-1}^\top\\
a_{1,-1} & A_{-1,-1}
\end{pmatrix},
\qquad
A'=\begin{pmatrix}
a_{11}' & a_{1,-1}'^\top\\
a_{1,-1}' & A_{-1,-1}
\end{pmatrix}.
\]
Because only row/column $1$ changes, the principal submatrix $A_{-1,-1}$ is identical
in $A$ and $A'$. By the Schur complement determinant formula,
\[
\det(A)=\det(A_{-1,-1})\cdot s,
\qquad
\det(A')=\det(A_{-1,-1})\cdot s',
\]
where
\[
s := a_{11}-a_{1,-1}^\top A_{-1,-1}^{-1}a_{1,-1},
\qquad
s' := a_{11}'-a_{1,-1}'^\top A_{-1,-1}^{-1}a_{1,-1}'.
\]
Therefore,
\[
\log\det(A')-\log\det(A)=\log s' - \log s.
\]
Since $\|\phi_1\|_2=\|\phi_1'\|_2=1$, we have $a_{11}=a_{11}'=1+\epsilon$. Furthermore, since $A$ is positive definite for $\epsilon > 0$, its principal submatrix $A_{-1,-1}$ and its corresponding inverse are also positive definite. Hence the
subtracted quadratic forms are nonnegative, implying that $s,s'\le 1+\epsilon$.
Also, $A\succeq \epsilon I_k$ implies $A^{-1}\preceq \epsilon^{-1}I_k$,
hence $(A^{-1})_{11}\le 1/\epsilon$. But the block inverse identity gives
$(A^{-1})_{11}=1/s$, so $s\ge \epsilon$; similarly $s'\ge \epsilon$.
Thus $s,s'\in[\epsilon,1+\epsilon]$, and
\[
\big|\log\det(A')-\log\det(A)\big|
=
|\log s' - \log s|
\le
\log\!\Big(\frac{1+\epsilon}{\epsilon}\Big).
\]
Dividing by $2k$ yields
\[
\big|\unadjmethod_\task-\unadjmethod_\task'\big|
\le
\frac{1}{2k}\log\!\Big(\frac{1+\epsilon}{\epsilon}\Big).
\]

Finally, combining the bounded difference for $\quadent_\task$ and $\unadjmethod_\task$ gives an overall bounded difference constant
\[
\big|\method_\task-\method_\task'\big|
\le
\frac{1}{k}\left(\alpha+\tfrac{1}{2}\log\!\Big(\frac{1+\epsilon}{\epsilon}\Big)\right)
= \frac{L}{k}.
\]
McDiarmid's inequality then yields \cref{eq:app:mcdiarmid}.

To show ranking consistency, Note that if $\mathbb{E}[\method_{\task_a}]-\mathbb{E}[\method_{\task_b}]\ge \Delta$
and $\method_{\task_a}\le \method_{\task_b}$ occurs, then either
$\method_{\task_a}-\mathbb{E}\method_{\task_a}\le -\Delta/2$ or
$\method_{\task_b}-\mathbb{E}\method_{\task_b}\ge \Delta/2$.
Applying \cref{eq:app:mcdiarmid} to each instance and union bounding gives
\cref{eq:app:ranking_mcdiarmid}.
\end{proof}

Note that if $(E[V_t\mid \query_t\in W] > E[V_t\mid \query_t\in C])$ and both concentrate with $k$, then ranking errors decay exponentially in $k$, providing support for the strong empirical performance of \algname in meeting desiderata \ref{r:classify}.

\section{Discussion}
\label{app:insight}

\subsection{Algorithm}

We present the algorithm for \algname in \cref{alg:umpire}.

\begin{algorithm}[h]
   \caption{UMPIRE for task instance $t$}
   \label{alg:umpire}
\begin{algorithmic}
   \STATE {\bfseries Input:} Model $\mathcal{M}$, Query $q_t$, Num. Samples $k$, Jitter $\epsilon$, Scaling $\alpha$
   \STATE {\bfseries Output:} Uncertainty Score $V_t$
   \STATE Initialize $\Phi \in \mathbb{R}^{k \times d}$, $\mathbf{p} \in \mathbb{R}^k$
   \FOR{$i=1$ {\bfseries to} $k$}
       \STATE Sample response $y_i \sim \mathcal{M}(\cdot | q_t)$ with Temp. $T=1$
       \STATE Extract embedding $\phi_i \leftarrow \text{Normalize}(\text{LastLayer}(y_i))$
       \STATE Compute probability $p_i \leftarrow P_{\mathcal{M}}(y_i|q_t)$
   \ENDFOR
   \STATE Compute Gram matrix $G \leftarrow \Phi \Phi^\top$ from $\{\phi_i\}$
   \STATE $U_t \leftarrow \frac{1}{2k} \log\det(G + \epsilon I_k)$ \COMMENT{Semantic volume}
   \STATE $Q_t \leftarrow \frac{1}{k} \sum_{i=1}^k (1 - p_i)$ \COMMENT{Incoherence score}
   \STATE $V_t \leftarrow U_t + \alpha Q_t$ \COMMENT{\cref{eq:casev}}
   \STATE \textbf{return} $V_t$
\end{algorithmic}
\end{algorithm}

\subsection{Efficiency analysis}
\label{app:efficiency}

To compare the efficiency (\ref{r:efficiency}) of \algname and the baselines, we analyze the overall computational running time that the methods incur on top of the MLLM inference costs due to sampling responses (since all baseline methods involve the same response sampling process). Experiments are conducted on a single L40 involving the processing of 3000 VQAv2 task instances. In \Cref{fig:run_time_num_gen}(a), we plot the methods' running time overheads of uncertainty metrics against their AUROC performance.  We observe that methods relying on external semantic evaluation, such as \se, \cu, incur a prohibitive computational cost (high overhead) due to the heavy usage of NLI models or clustering algorithms. In contrast, \algname (marked by the red star) operates with negligible overhead, comparable to simple entropy-based baselines like \lne, while achieving state-of-the-art accuracy. Moreover, by leveraging the \texttt{fast\_logdet} library, \algname computes only the log-determinant of the matrix rather than its full eigenvalue decomposition (\eigen), further reducing computational overhead. This places \algname on the optimal Pareto frontier, offering a high-performance solution that is computationally lightweight enough for real-time applications.

Furthermore, we analyze the sensitivity of our method to the number of sampled responses $k$ in \Cref{fig:run_time_num_gen}(b) (setup details in \Cref{app:num_gen}). \algname consistently outperforms all baselines across the entire range of $k$ and achieves strong performance even with few generations.

\textbf{End-to-end cost and small sampling budgets.} Since \Cref{fig:run_time_num_gen}(a) only reports the overhead on top of response sampling, we also measure the end-to-end wall-clock time per instance. This includes response generation with vLLM batched inference \citep{kwonEfficientMemoryManagement2023} and embedding extraction with a Hugging Face model, both on a single L40 GPU. \Cref{tab:cost_perdataset} compares \algname with $k\in\{2,3,5\}$ against single-pass baselines, including Margin Confidence and Verbalized Confidence \citep{xiong2024llmsexpressuncertaintyempirical}. With only $k=2$, \algname already achieves higher AUROC than the best single-pass baseline on every dataset. \Cref{tab:cost_k} further reports the cost and AUROC of \algname on VQAv2 as $k$ increases. AUROC saturates around $k=5$, while the cost grows roughly linearly with $k$.

\begin{figure}
    \centering
    \includegraphics[width=\linewidth]{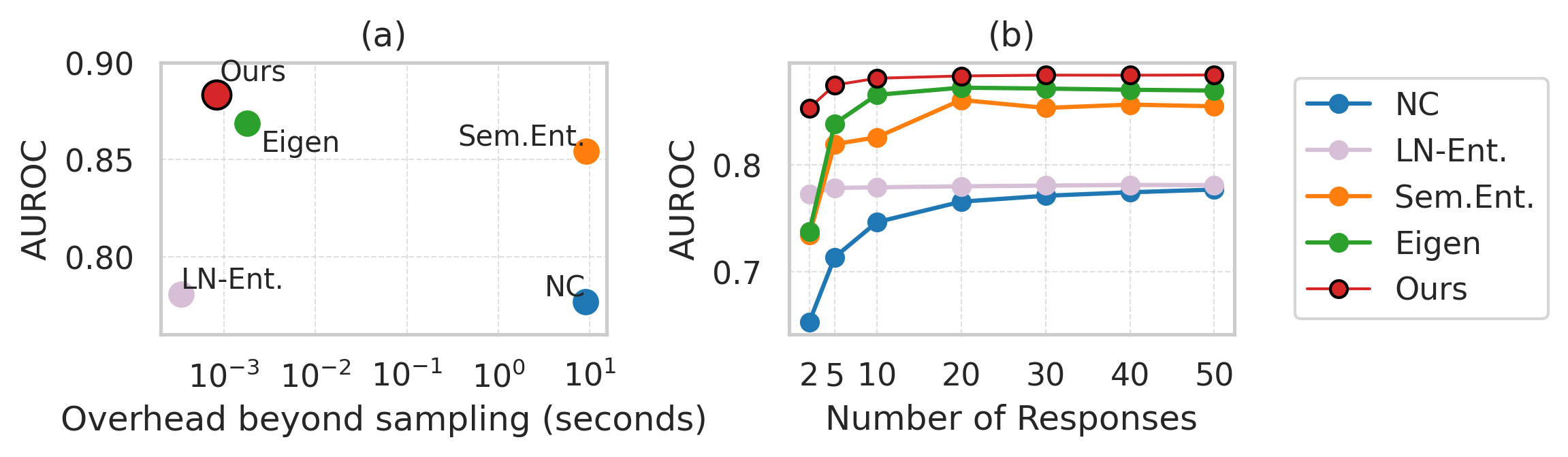}
    \caption{Efficiency analysis of \algname compared to baselines on VQAv2 (first 3k validation instances). \textbf{(a)} Computational overhead beyond response sampling (seconds) versus uncertainty estimation performance (AUROC at $k=50$). \algname{} (red) achieves state-of-the-art performance with negligible overhead, avoiding the high computational cost associated with semantic equivalence checks in methods like \se. End-to-end cost including generation is reported in \Cref{tab:cost_perdataset,tab:cost_k}. \textbf{(b)} The effect of the number of generated responses $k$ on performance. \algname consistently outperforms other methods across all sample sizes and converges to high AUROC even with few generations (e.g., $k=5$), demonstrating superior sample efficiency.}
    \label{fig:run_time_num_gen}
\end{figure}

\begin{table}[h]
\centering\small
\caption{AUROC and end-to-end wall-clock time (s per instance, including generation; single L40, vLLM) for \algname at small sampling budgets versus single-pass baselines.}
\label{tab:cost_perdataset}
\setlength{\tabcolsep}{3pt}
\begin{tabular}{lcccccccc}
\toprule
& \multicolumn{2}{c}{VQAv2} & \multicolumn{2}{c}{AdVQA} & \multicolumn{2}{c}{OKVQA} & \multicolumn{2}{c}{Avg.} \\
\cmidrule(lr){2-3}\cmidrule(lr){4-5}\cmidrule(lr){6-7}\cmidrule(lr){8-9}
Method & AUROC & Time & AUROC & Time & AUROC & Time & AUROC & Time \\
\midrule
\algname ($k=2$) & .849 & 0.56 & .746 & 0.53 & .721 & 0.57 & .772 & 0.55 \\
\algname ($k=3$) & .861 & 0.71 & .766 & 0.67 & .732 & 0.72 & .786 & 0.70 \\
\algname ($k=5$) & \textbf{.867} & 1.04 & \textbf{.778} & 0.96 & \textbf{.741} & 1.04 & \textbf{.795} & 1.01 \\
\midrule
Max Seq.\ Prob. & .633 & 0.23 & .591 & 0.23 & .581 & 0.27 & .602 & 0.24 \\
Perplexity & .630 & 0.23 & .588 & 0.23 & .580 & 0.27 & .599 & 0.24 \\
Mean Tok.\ Entropy & .629 & 0.23 & .586 & 0.23 & .580 & 0.27 & .598 & 0.24 \\
Margin Conf. & .615 & 0.23 & .594 & 0.23 & .545 & 0.27 & .585 & 0.24 \\
Verb.\ Conf. & .572 & 0.43 & .510 & 0.43 & .522 & 0.47 & .535 & 0.44 \\
\bottomrule
\end{tabular}
\end{table}

\begin{table}[h]
\centering\small
\caption{\algname end-to-end cost and AUROC versus sampling budget $k$.}
\label{tab:cost_k}
\begin{tabular}{lcccccc}
\toprule
$k$ & 2 & 3 & 5 & 10 & 20 & 50 \\
\midrule
Cost (s per instance) & 0.556 & 0.713 & 1.036 & 1.838 & 3.480 & 8.383 \\
AUROC & .849 & .861 & .867 & .872 & .873 & .876 \\
\bottomrule
\end{tabular}
\end{table}

\subsection{Assessing multimodal query input coherence (\ref{r:coherence})}
\label{app:coherence_exp}

For a metric to satisfy \ref{r:coherence}, it should consider the coherence of each sampled response with respect to the multimodal task instance query, rather than just a single modality (e.g., text). We design an experimental setting on image-text modalities to assess this by computing uncertainty metrics based on (1) both image and text portions of the query $\query_\task = (\img_\task, \query_{\text{text}})$, (2) image that has additive noise $\mathcal{N}(0, 0.5^2)$, $\query_{\task, \text{noise}} = (\img_\text{noise}, \query_{\text{text}})$,  (3) image that is entirely black, $\query_{\task, \text{black}} = (\img_{\text{black}}, \query_{\text{text}})$, and (4) no image, i.e., only the text portion of the query, $\query_{\task, \text{text}} = \query_{\text{text}})$ (image removed). A metric that satisfies \ref{r:coherence} should perform significantly better under (1), while a metric that does not will produce similar performance regardless of (1)-(4).

Specifically, for (2)-(4), after the MLLM has generated responses $y_\task$ based on $\query_\task$, we recompute the various metrics \lne, \se, \eigen, and \algname based on the query-answer pair lower quality to no image, e.g., based on recomputing the response logits and embedding vectors of text-only query-answer pairs $[\query_{\text{text}},y_\task]$, on the first 3000 samples of the VQAv2 validation set. 
In \Cref{fig:r3_r4} (b), we observe that \lne and \algname, and to a smaller extent \se, are sensitive to the lack of multi-modality information, with their performance increasing once the image queries are provided during the computation of the metrics. On the other hand, \eigen is insensitive to whether the image query is provided or not. This may be because \eigen measures only the diversity of responses through the covariance matrix of text response sentence embeddings across multiple generations, which is not affected by the image query bias. On the contrary, logit signals are more sensitive to the coherence of the multimodal input query and the generated response, hence metrics that use some form of that such as \lne and \algname can better satisfy \ref{r:coherence}.

\paragraph{Impact of noisy inputs on UQ performance.}
In \ref{r:coherence}, we evaluated the metrics' sensitivity to input corruptions while holding the MLLM output fixed. To complement this, we now evaluate the metrics in a setting where the MLLM output is allowed to vary. Varying the output alongside corrupted inputs tests a different practical aspect: the robustness of UQ metrics when the MLLM itself is challenged by noise. As shown in \cref{tab:noise_corruption}, \algname continues to outperform the baselines. While comparing absolute scores is difficult due to differing metric scales, \cref{tab:uncertainty_score_change} demonstrates that uncertainty scores appropriately increase with corrupted inputs, as expected. Furthermore, \eigen remains relatively insensitive to noise, corroborating the evidence from our previous fixed-output setting (\ref{r:coherence}).

\begin{table}[h]
\caption{Under input image noise corruption with varying MLLM output.}
\label{tab:noise_corruption}
\centering
\begin{tabular}{l cccc}
\toprule
AUROC & \se & \lne & \eigen & \algname \\
\midrule
Normal & 0.836 & 0.784 & 0.863 & \textbf{0.876} \\
Noise  & 0.786 & 0.678 & 0.795 & \textbf{0.804} \\
\bottomrule
\end{tabular}
\end{table}

\begin{table}[h]
\caption{Averaged uncertainty score (\% change) for normal and noise/black image corruption. \eigen is relatively insensitive to noise.}
\label{tab:uncertainty_score_change}
\centering
\begin{tabular}{l cccc}
\toprule
 & \se & \lne & \eigen & \algname \\
\midrule
Normal & 0.329 & 0.078 & -7.608 & 0.886 \\
Noise  & 0.399 ($\uparrow$21.3\%) & 0.097 ($\uparrow$23.2\%) & \textit{-7.523 ($\uparrow$1.1\%)} & 5.298 ($\uparrow$497.6\%) \\
Black  & 0.443 ($\uparrow$34.6\%) & 0.095 ($\uparrow$21.1\%) & \textit{-7.395 ($\uparrow$2.8\%)} & 7.209 ($\uparrow$713.2\%) \\
\bottomrule
\end{tabular}
\end{table}

\subsection{Comparisons with \eigen}
\label{app:eigen}

As mentioned in \cref{app:benchmark}, the \eigen \citep{chenLLMsInternalStates2024} metric involves computing the log determinant of the covariance matrix of sampled sentence embeddings. At first glance, this metric may seem similar to that of \algname. However, there are key differences that lead to \eigen consistently underperforming our proposed \algname metric, as can be seen in our empirical results in \cref{sec:exp},
and \eigen could be interpreted as a special case of \algname.

A major distinction, among others, is that \citet{chenLLMsInternalStates2024} analyzed only the LLM setting, and proposed \eigen by considering the differential entropy of sentence embeddings, assuming that the embeddings form a multivariate Gaussian distribution -- this motivated the log determinant operation in the metric, which \algname also contains. Note that following their theoretical motivations, the log determinant term of \eigen will be on the sample covariance (of size $d$ by $d$), where the feature space dimensionality of each sample (LLM embeddings) is typically very large.  However, our \algname framework considers the more general MLLM setting, and adopts a different approach inspired by the quality-diversity kernel decomposition of determinantal point processes (DPP), which naturally factors the incoherence scores when computing the \algname metric to adjust the semantic volume enclosed by the responses' semantic embeddings, computed by the log determinant of the Gram matrix (of size $k$ by $k$). This inclusion of the incoherence scores, and the resulting quadratic entropy term $\quadent$, help (1) satisfy \ref{r:coherence}, as we can see in \cref{app:coherence_exp} that \eigen does not, and (2) significantly improve metric performance (\cref{app:alphasensitivity}). 

\subsection{Additional discussion on the effectiveness desiderata}
\label{app:effectdesiderata}

In this section, we provide further discussion on the various effectiveness desiderata, such as the differences and relevance of \ref{r:classify}, \textbf{\ref{r:r2a}} and \textbf{\ref{r:r2b}}. For ease of discussion, we focus on comparing \ref{r:classify} and \textbf{\ref{r:r2b}}, which is a stricter form of \textbf{\ref{r:r2a}}.

The discrimination desiderata \ref{r:classify} and the risk-score desiderata \textbf{\ref{r:r2b}} are primarily motivated by different considerations. In the former, we are concerned about classifying whether a task instance $t$ will be answered correctly or not by the MLLM. As represented in \cref{eq:distinguish}, for this desiderata, the metric should be able to successfully rank the task instances that the MLLM will get wrong higher than those that it will get correct, which can be evaluated by the AUROC of the metric. Such evaluations are used in many LLM uncertainty quantification works \citep{farquharDetectingHallucinationsLarge2024,malinin2021uncertainty,chenLLMsInternalStates2024,xiong2024llmsexpressuncertaintyempirical} to assess the performance of their metrics. While useful, note that the desiderata does not consider a quantitative, continuous measure of the uncertainty associated with each task response, since discrimination of correct/wrong responses is a binary task.  

However, in the latter, we are concerned about providing an accurate, calibrated estimate of whether the MLLM will get a task instance wrong, conditional on the uncertainty metric (as in \cref{eq:r2b_calibration_clean}), which can be evaluated via the expected calibration error (ECE). Note that in this scenario, we are not concerned about classifying whether a task instance will be answered correctly (\ref{r:classify}), but instead are focused on being accurate about the \emph{probability} that a task instance will be answered wrongly given an associated metric value. 

To illustrate the difference, consider an extreme example where an MLLM will definitely get $50\%$ of the task instances correct, and the rest wrong. The vacuous metric that assigns the same uncertainty score to all task instances might satisfy \textbf{\ref{r:r2b}} since it will output the average error rate, 0.5, as the score for all task instances. This metric would violate \ref{r:classify} and fail to classify the correct from the wrong task instances. Instead, a better metric might strive to assign $1$ to all task instances that can be answered correctly and $0$ to the rest, satisfying both \ref{r:classify} and \textbf{\ref{r:r2b}}.

In practice, we would likely not have perfect information prior to evaluation on whether a task instance will be correct or wrong. That is why for \ref{r:classify} the goal is only for the metric's AUROC to get as close to 1 as possible, as the best possible AUROC would depend on the model and task. However, given two metrics that can achieve the same AUROC, a poor metric might only obtain the right relative ranking of task instances, while a good metric would not only achieve the same AUROC but also provide calibrated probabilities on how likely a task instance would be answered correctly or not. Hence, both the \ref{r:classify} and \textbf{\ref{r:r2b}} should be considered when evaluating uncertainty metrics, as we described in \cref{sec:desiderata}. In the absence of a small development set of unlabeled task instances before deployment, metrics satisfying \textbf{\ref{r:r2a}} would at least provide interpretable relative information regarding how likely a task instance would be answered correctly compared to another. 

We also consider Spearman rank correlation and Pearson correlation without binning (point-biserial Pearson). Note that Spearman correlation and Biserial Pearson focus more on discrimination (similar to \ref{r:classify}), rather than the metric's structural property (linearity) like \ref{r:r2a}. As shown in \cref{tab:correlation_metrics}, \algname also outperforms baselines on Spearman rank correlation and point-biserial Pearson correlation.

\begin{table}[ht]
\caption{Spearman and point-biserial Pearson (bi-Pearson) correlations of various UQ metrics across image-text datasets ($\uparrow$ better). Overall, \algname outperforms baselines.}
\label{tab:correlation_metrics}
\centering
\setlength{\tabcolsep}{3pt}
\resizebox{\linewidth}{!}{%
\begin{tabular}{ll ccccc >{\columncolor[gray]{0.9}}c}
\toprule
\textbf{Metric} & \textbf{Method} & VQAv2 & OKVQA & AdVQA & MathVista & VQA-RAD & \textbf{Avg.} \\
\midrule

\multirow{5}{*}{\rotatebox[origin=c]{90}{Spearman $\uparrow$}}
 & NC & 0.857{\scriptsize \textsubscript{±0.026}} & 0.374{\scriptsize \textsubscript{±0.298}} & 0.528{\scriptsize \textsubscript{±0.123}} & 0.609{\scriptsize \textsubscript{±0.095}} & 0.414{\scriptsize \textsubscript{±0.078}} & 0.556 \\
 & Sem.Ent. & 0.857{\scriptsize \textsubscript{±0.040}} & 0.522{\scriptsize \textsubscript{±0.040}} & 0.844{\scriptsize \textsubscript{±0.005}} & 0.874{\scriptsize \textsubscript{±0.038}} & 0.574{\scriptsize \textsubscript{±0.051}} & 0.734 \\
 & LN-Ent. & 0.570{\scriptsize \textsubscript{±0.143}} & 0.749{\scriptsize \textsubscript{±0.109}} & 0.881{\scriptsize \textsubscript{±0.057}} & \underline{0.880{\scriptsize \textsubscript{±0.059}}} & \underline{0.741{\scriptsize \textsubscript{±0.043}}} & 0.764 \\
 & Eigen & \underline{0.944{\scriptsize \textsubscript{±0.015}}} & \underline{0.884{\scriptsize \textsubscript{±0.059}}} & \underline{0.936{\scriptsize \textsubscript{±0.019}}} & 0.803{\scriptsize \textsubscript{±0.038}} & 0.570{\scriptsize \textsubscript{±0.056}} & \underline{0.827} \\
 & \textbf{Ours} & \textbf{0.990{\scriptsize \textsubscript{±0.002}}} & \textbf{0.976{\scriptsize \textsubscript{±0.003}}} & \textbf{0.984{\scriptsize \textsubscript{±0.003}}} & \textbf{0.885{\scriptsize \textsubscript{±0.027}}} & \textbf{0.767{\scriptsize \textsubscript{±0.055}}} & \textbf{0.921} \\

\midrule

\multirow{5}{*}{\rotatebox[origin=c]{90}{bi-Pearson $\uparrow$}}
 & NC & 0.308{\scriptsize \textsubscript{±0.006}} & 0.062{\scriptsize \textsubscript{±0.013}} & 0.235{\scriptsize \textsubscript{±0.012}} & 0.387{\scriptsize \textsubscript{±0.003}} & 0.373{\scriptsize \textsubscript{±0.016}} & 0.273 \\
 & Sem.Ent. & \underline{0.522{\scriptsize \textsubscript{±0.010}}} & 0.356{\scriptsize \textsubscript{±0.001}} & \underline{0.456{\scriptsize \textsubscript{±0.003}}} & \underline{0.456{\scriptsize \textsubscript{±0.011}}} & 0.430{\scriptsize \textsubscript{±0.007}} & \underline{0.444} \\
 & LN-Ent. & 0.361{\scriptsize \textsubscript{±0.007}} & 0.339{\scriptsize \textsubscript{±0.003}} & 0.279{\scriptsize \textsubscript{±0.002}} & 0.264{\scriptsize \textsubscript{±0.004}} & 0.282{\scriptsize \textsubscript{±0.012}} & 0.305 \\
 & Eigen & \textbf{0.530{\scriptsize \textsubscript{±0.011}}} & \underline{0.375{\scriptsize \textsubscript{±0.003}}} & 0.447{\scriptsize \textsubscript{±0.006}} & 0.401{\scriptsize \textsubscript{±0.006}} & \underline{0.457{\scriptsize \textsubscript{±0.003}}} & 0.442 \\
 & \textbf{Ours} & 0.517{\scriptsize \textsubscript{±0.009}} & \textbf{0.420{\scriptsize \textsubscript{±0.003}}} & \textbf{0.493{\scriptsize \textsubscript{±0.005}}} & \textbf{0.479{\scriptsize \textsubscript{±0.007}}} & \textbf{0.524{\scriptsize \textsubscript{±0.005}}} & \textbf{0.486} \\

\bottomrule
\end{tabular}
}
\end{table}

\subsection{Role of $\unadjmethod$ and $\quadent$ in \algname~: Hyperparameter $\alpha$ sensitivity}
\label{app:alphasensitivity}

As mentioned in \textbf{U3}, \Cref{sec:method} and in the practical consideration subsection in \Cref{sec:prac_consi}, \algname has a hyperparameter $\alpha$ that controls the balance between two terms in \Cref{eq:casev}: the unadjusted semantic volume $\unadjmethod$ and the quadratic entropy term $\quadent$.

\paragraph{Role of $\unadjmethod$ and $\quadent$.} By varying this hyperparameter, we can observe the contribution of these two terms to the final performance of \algname. When $\alpha = 0$, \algname is equivalent to $\unadjmethod$, and increasing $\alpha$ from 0 increases the contribution of $\quadent$. 
As discussed in \cref{sec:prac_consi}, $\quadent$ and $\unadjmethod$ play complementary roles. On its own, $\quadent$ is a fairly strong discriminator (\ref{r:classify}, achieving an average AUROC of 81.2), but combining $\unadjmethod$ with it to form the overall \algname metric $\method$ results in significantly better risk-score quality (\ref{r:risk}) while maintaining competitive discrimination capabilities (\ref{r:classify}). 
We see this in \cref{tab:adaptive_best} -- while \algname maintains a competitive AUROC comparable to the strong $Q$ baseline, it achieves a dramatic improvement in reliability metrics: the CPC rises significantly from 84.4 ($Q$) to 91.7 ($\method_{\text{best}}$), and the ECE is reduced by over 30\% (from 0.071 to 0.048). 

To validate this synergy further, we performed a Likelihood Ratio Test (LRT) to assess whether the unadjusted semantic volume ($\unadjmethod$) contributes distinct information not captured by incoherence ($\quadent$). We evaluated a restricted model ($H_0: \acc(\model,\task) = \sigma(\beta_0 + \beta_1 Q)$) relying solely on incoherence against a full model ($H_1: \acc(\model,\task) = \sigma(\beta_0 + \beta_1 Q + \beta_2 U$)) that incorporates both terms, computing the test statistic $\Lambda = -2(\ell_{\text{restricted}} - \ell_{\text{full}})$ to assess whether the addition of $U$ yields a statistically significant improvement in model fit. As shown in \Cref{tab:lrt_results}, the inclusion of $\unadjmethod$ yields a statistically significant improvement in model fit across all VQA datasets ($p < 0.001$).

\begin{table}[ht]
\caption{Likelihood Ratio Test (LRT) Results. Comparison of model fit between using Incoherence ($Q$) alone versus the Full Model ($Q+U$). Significance (yes/no) denotes $p < 0.001$.}
\label{tab:lrt_results}
\centering
\scriptsize
\resizebox{0.6\linewidth}{!}{
\begin{tabular}{lrrr}
\toprule
\textbf{Dataset} & \textbf{LR Statistic ($\chi^2$)} & \textbf{\textit{p}-value} & \textbf{Significance?} \\
\midrule
VQAv2 & 165.30 & $<0.0001$ & yes \\
OKVQA & 41.49 & $<0.0001$ & yes \\
AdVQA & 526.33 & $<0.0001$ & yes \\
VQA-RAD & 30.57 & $<0.0001$ & yes \\
MathVista & 32.22 & $<0.0001$ & yes \\
\bottomrule
\end{tabular}
}
\end{table}

To better visualize the interplay between $\quadent$ and $\unadjmethod$, we can vary $\alpha$ and examine the change in \algname's performance in AUROC, ECE, and CPC. As can be seen in \Cref{fig:alpha_tunning}, we see that the performance for AUROC is relatively stable across \algname once we have included a sizeable mix of $\quadent$. This is due to AUROC's nature as a rank-based metric invariant to strictly monotonic transformations, and adjusting $\alpha$ still broadly preserves the relative ordering of uncertainty scores between correct and incorrect instances. In contrast, CPC and ECE are distribution-sensitive and they exhibit higher sensitivity to $\alpha$ as they rely on the precise calibration and linearity of the metric magnitudes In this case, we observe `U' shape curves indicative of both terms playing an important role in good performance, as we have seen via \cref{tab:adaptive_best} and \cref{tab:lrt_results}.

\begin{figure}[ht]
    \centering
    \includegraphics[width=\linewidth]{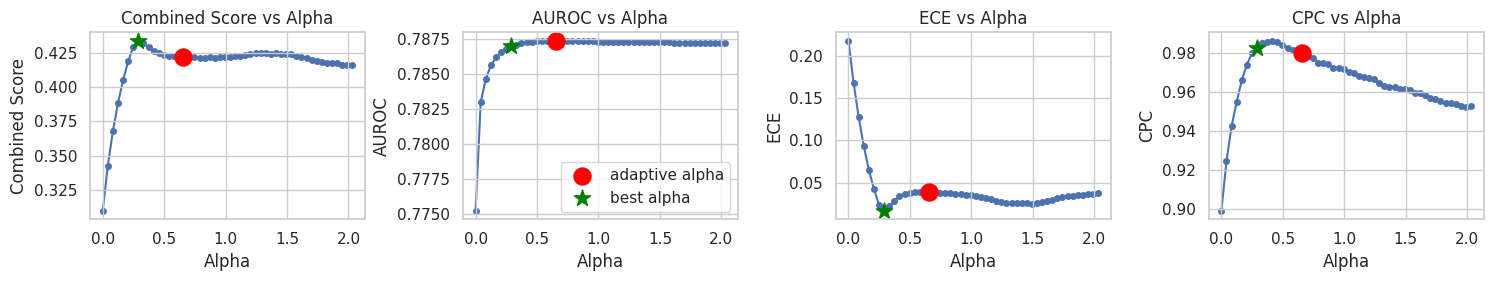}
    \caption{
    Sensitivity analysis of the hyperparameter $\alpha$ in \algname on the AdVQA dataset. The plots illustrate the effect of the incoherence score weight on Combined Score, AUROC, ECE, and CPC. The value $\alpha=0$ corresponds to the unadjusted semantic volume. The green star (\textcolor{green}{$\star$}) marks the optimal $\alpha$ found by maximizing the Combined Score, while the red dot (\textcolor{red}{$\bullet$}) indicates the value selected by our adaptive strategy. Notably, the adaptive approach yields performance close to the optimum without requiring labeled data for tuning.}
    \label{fig:alpha_tunning}
\end{figure}

\begin{table*}[ht]
\caption{
We compare the unadjusted semantic volume \textbf{$\unadjmethod$}, Monte Carlo estimate of the expected incoherence value or quadratic entropy \textbf{$\quadent$}, and our final form of \algname \textbf{$\method$}. We report two values for $\method$: $\method_{adap}$ which uses $\alpha$ based on the adaptive $\alpha$ heuristic that does not require labeled data, and $\method_{best}$ which uses $\alpha$ found based on a held-out labeled development set.
}
\label{tab:adaptive_best}
\centering\scriptsize
\setlength{\tabcolsep}{1.1pt}
\resizebox{\linewidth}{!}{%
\begin{tabular}{l|cccc|cccc|cccc}
\toprule
Dataset & \multicolumn{4}{c|}{AUROC $\uparrow$}& \multicolumn{4}{c|}{CPC $\uparrow$}& \multicolumn{4}{c}{ECE $\downarrow$} \\
\cmidrule(lr){2-5} \cmidrule(lr){6-9} \cmidrule(lr){10-13}
 & $\unadjmethod$ & $\quadent$ & $\method_{\text{adap.}}$ & $\method_{\text{best}}$ & $\unadjmethod$ & $\quadent$ & $\method_{\text{adap.}}$ & $\method_{\text{best}}$ & $\unadjmethod$ & $\quadent$ & $\method_{\text{adap.}}$ & $\method_{\text{best}}$\\
\midrule\midrule
\multicolumn{13}{l}{\textbf{Image-Text}} \\ VQAv2 & 86.7\textsubscript{±0.2} & \textbf{88.2\textsubscript{±0.2}} & \underline{88.1\textsubscript{±0.2}} & 88.0\textsubscript{±0.1} & \underline{95.4\textsubscript{±0.6}} & 88.0\textsubscript{±0.5} & 94.6\textsubscript{±0.1} & \textbf{97.5\textsubscript{±1.5}} & .051\textsubscript{±.007} & .070\textsubscript{±.005} & \underline{.038\textsubscript{±.004}} & \textbf{.028\textsubscript{±.011}} \\ OKVQA & 73.7\textsubscript{±0.1} & \textbf{75.6\textsubscript{±0.1}} & \underline{75.3\textsubscript{±0.1}} & 75.2\textsubscript{±0.1} & 85.1\textsubscript{±8.6} & 92.4\textsubscript{±0.4} & \textbf{96.7\textsubscript{±0.7}} & \underline{93.5\textsubscript{±6.5}} & .163\textsubscript{±.011} & .127\textsubscript{±.002} & \underline{.045\textsubscript{±.007}} & \textbf{.041\textsubscript{±.002}} \\ AdVQA & \underline{77.5\textsubscript{±0.3}} & \textbf{78.7\textsubscript{±0.2}} & \textbf{78.7\textsubscript{±0.2}} & \textbf{78.7\textsubscript{±0.3}} & 89.9\textsubscript{±1.6} & 92.4\textsubscript{±0.4} & \textbf{98.0\textsubscript{±0.1}} & \underline{97.7\textsubscript{±0.7}} & .218\textsubscript{±.004} & .083\textsubscript{±.003} & \underline{.039\textsubscript{±.003}} & \textbf{.034\textsubscript{±.005}} \\ MathVista & 82.1\textsubscript{±0.7} & \textbf{83.1\textsubscript{±0.6}} & 82.6\textsubscript{±0.6} & \underline{82.8\textsubscript{±0.5}} & 71.4\textsubscript{±1.3} & 79.1\textsubscript{±5.2} & \underline{83.9\textsubscript{±2.7}} & \textbf{85.7\textsubscript{±4.9}} & .308\textsubscript{±.013} & \textbf{.029\textsubscript{±.006}} & .089\textsubscript{±.015} & \underline{.083\textsubscript{±.035}} \\ VQA-RAD & 79.6\textsubscript{±0.6} & 80.4\textsubscript{±0.2} & \textbf{80.7\textsubscript{±0.3}} & \underline{80.6\textsubscript{±0.3}} & 61.2\textsubscript{±0.6} & \underline{84.4\textsubscript{±6.3}} & 80.8\textsubscript{±7.1} & \textbf{88.4\textsubscript{±4.8}} & .352\textsubscript{±.020} & \textbf{.045\textsubscript{±.013}} & .072\textsubscript{±.009} & \underline{.061\textsubscript{±.005}} \\ \midrule \multicolumn{13}{l}{\textbf{Audio}} \\ SLUE & 78.9\textsubscript{±0.4} & \textbf{83.0\textsubscript{±0.3}} & 82.2\textsubscript{±0.3} & \underline{82.5\textsubscript{±0.4}} & 74.4\textsubscript{±3.1} & 89.2\textsubscript{±1.1} & \textbf{93.6\textsubscript{±0.5}} & \underline{92.4\textsubscript{±1.6}} & .185\textsubscript{±.001} & .066\textsubscript{±.003} & \underline{.054\textsubscript{±.009}} & \textbf{.050\textsubscript{±.007}} \\ SpokenSQ. & 77.7\textsubscript{±0.1} & \textbf{80.2\textsubscript{±0.1}} & \underline{79.7\textsubscript{±0.1}} & \underline{79.7\textsubscript{±0.2}} & 91.7\textsubscript{±0.7} & 91.8\textsubscript{±1.4} & \textbf{97.8\textsubscript{±0.5}} & \underline{97.7\textsubscript{±0.5}} & .239\textsubscript{±.005} & \underline{.079\textsubscript{±.001}} & \textbf{.031\textsubscript{±.002}} & \textbf{.031\textsubscript{±.002}} \\ \addlinespace \multicolumn{13}{l}{\textbf{Video}} \\ VidMME & 81.2\textsubscript{±0.9} & \underline{81.5\textsubscript{±0.1}} & \textbf{81.9\textsubscript{±0.4}} & \textbf{81.9\textsubscript{±0.2}} & 75.7\textsubscript{±3.5} & 57.7\textsubscript{±7.6} & \underline{77.8\textsubscript{±4.3}} & \textbf{80.3\textsubscript{±6.4}} & .328\textsubscript{±.011} & \underline{.071\textsubscript{±.003}} & .126\textsubscript{±.026} & \textbf{.055\textsubscript{±.010}} \\ \addlinespace \midrule \rowcolor[gray]{0.9}\textbf{Avg (all)} & 79.7 & \textbf{81.3} & \underline{81.2} & \underline{81.2} & 80.6 & 84.4 & \underline{90.4} & \textbf{91.7} & .230 & .071 & \underline{.062} & \textbf{.048} \\
\bottomrule
\end{tabular}
}

\end{table*}

\paragraph{Choosing hyperparameter  $\alpha$.} Unless stated otherwise, in our experiments, we did not tune the hyperparameter based on a labeled development set but instead set $\alpha$ such that both terms have the same expected contribution (i.e., `adaptive' $\alpha$ set to be the ratio of the median of $\quadent$ and $\unadjmethod$ based on an unlabeled sample of task instances). We also show that \algname performance is also robust to the size of this unlabeled set of task instances that we use to set by conducting an experiment of using different subset sizes (1\%, 5\%, and 10\%) of the unlabeled evaluation set to set $\alpha$. As can be seen in the table \Cref{tab:adaptive_alpha_subset}, \algname maintains consistent and high performance across all metrics, with minimal variation in AUROC, ECE, CPC, and AURAC, underscoring its robustness to the subset of unlabeled data used.

However, in practice, users could potentially search for a better hyperparameter value for their task, such as via grid search or AutoML methods like Bayesian Optimization on a small labeled development set. To determine the optimal operating point during such a search, we can utilize a Combined Score formulated as: $\text{Combined Score} = w_1 \cdot \text{AUROC} + w_2 \cdot \text{CPC} - w_3 \cdot \text{ECE}$. In \Cref{fig:alpha_tunning}, we provide an illustration of these performance metrics versus $\alpha$ values calculated on the full AdVQA test set. While \Cref{tab:adaptive_best} reports the specific performance gains achieved by tuning $\alpha$ on a 10\% labeled development set, the results in this table and \Cref{fig:alpha_tunning} both demonstrates that the `adaptive alpha' approach (red dot), which balances both terms in \Cref{eq:casev} without supervision, is not very far off from the ideal optimum (marked by the green star).

\begin{table}[h]
\caption{\algname performance is robust to the size of unlabeled set of task instance that we use to compute adaptive $\alpha$. Results are computed on VQAv2 validation set.}
\label{tab:adaptive_alpha_subset}
\centering \scriptsize
\resizebox{0.5\columnwidth}{!}{%
\begin{tabular}{lllll}
\toprule
Subset Size & AUROC & ECE   & CPC   & AURAC \\
\midrule
\midrule
1\%          & 0.882 & 0.04  & 0.948 & 0.966 \\
5\%          & 0.882 & 0.038 & 0.947 & 0.966 \\
10\%         & 0.882 & 0.038 & 0.946 & 0.966 \\
\bottomrule
\end{tabular}
}
\end{table}

\subsection{Supervised post-hoc calibration}
\label{app:platt}
The ECE in \Cref{sec:calibration} is computed on min-max normalized scores without labels (\ref{r:r2b}). When a small labeled set is available, standard post-hoc calibration can also be applied to any uncertainty metric to obtain error probability estimates. We fit Platt scaling on a random $20\%$ labeled split, compute ECE on the remaining $80\%$, and report the mean over $10$ random seeds. As shown in \Cref{tab:platt}, \algname achieves the lowest ECE on VQAv2 and AdVQA by large margins, and the lowest mean ECE across the five datasets. The combined score $\method$ also achieves lower ECE than $\unadjmethod$ or $\quadent$ alone after calibration (\Cref{tab:platt_uq}), consistent with the complementarity of the two terms shown in \Cref{fig:alpha_tunning} and \Cref{tab:lrt_results}.

\begin{table}[h]
\centering\small
\caption{ECE ($\downarrow$) of Platt-scaled error probabilities (20\% labeled calibration split, mean over 10 seeds).}
\label{tab:platt}
\begin{tabular}{lcccccc}
\toprule
Method & VQAv2 & AdVQA & OKVQA & MathVista & VQA-RAD & Mean \\
\midrule
\cu & .028 & .048 & \textbf{.017} & .051 & .069 & .043 \\
\se & .033 & .018 & .023 & \textbf{.038} & .066 & .036 \\
\lne & .040 & .057 & .029 & .052 & .080 & .052 \\
\eigen & .045 & .043 & .042 & .054 & .062 & .049 \\
\algname{} & \textbf{.008} & \textbf{.012} & .019 & .044 & \textbf{.060} & \textbf{.029} \\
\bottomrule
\end{tabular}
\end{table}

\begin{table}[h]
\centering\small
\caption{ECE(Platt) ($\downarrow$) of $\unadjmethod$, $\quadent$ and the combined score $\method$.}
\label{tab:platt_uq}
\begin{tabular}{lcccccc}
\toprule
Score & VQAv2 & AdVQA & OKVQA & MathVista & VQA-RAD & Mean \\
\midrule
$\unadjmethod$ & .047 & .045 & .043 & .054 & .060 & .050 \\
$\quadent$ & .013 & .034 & .029 & .046 & .070 & .038 \\
$\method$ & \textbf{.008} & \textbf{.012} & \textbf{.019} & \textbf{.044} & \textbf{.060} & \textbf{.029} \\
\bottomrule
\end{tabular}
\end{table}

\section{Experimental settings and other results}
\label{app:exp_setup_result}
\subsection{Benchmarks}
\label{app:benchmark}
\subsubsection{Datasets}
\label{app:dataset}
For our experiments, we utilize a diverse set of general multi-modality question-answering baseline datasets to ensure a comprehensive evaluation across different scenarios. Specifically, for image-text understanding,  we use VQAv2 \citep{balanced_vqa_v2}, OKVQA \citep{okvqa}, and AdVQA \citep{li2021adversarial}, which include challenging cases such as out-of-distribution and adversarial settings. 
Besides, we also try to use the domain-specific visual QA datasets, including VQA-RAD, a dataset of question-answer pairs on radiology images, and MathVista, a consolidated Mathematical reasoning baseline within visual contexts. 
We evaluate our method using the first 15,000 samples from the validation split of VQAv2, along with the full validation sets of OKVQA (5,000 samples) and AdVQA (10,000 samples), the test split of VQA-RAD \citep{vqarad} (450 samples), and MathVista \citep{mathvista} (testmini split - 1,000 samples). 

These datasets provide a robust test bed for assessing the effectiveness of our approach across different types of visual QA tasks. Besides image-text understanding, we also show the effectiveness of various uncertainty metrics on audio-text and video-text understanding via audio QA datasets, including the test split of SLUE-P2-SQA5 \citep{shon2022slue}, Spoken SQuAD \citep{li2018spoken}, and video QA datasets with Video-MME-short  \citep{fu2025video} (we convert multi-choice answers into free-form answers by taking the correct choice).

\subsubsection{Baselines} The details of each baseline are as follows.
\begin{itemize}
    \item  \textbf{Neighborhood Consistency (\cu) \citep{khan2024consistency}.} This method tries to examine the reliability of the model via the consistency of the model’s responses over the visual rephrased questions generated by a small proxy Visual Question Generation (VQG) model. We implement this method by training \texttt{BLIP} \citep{li2022blipbootstrappinglanguageimagepretraining} as the VQG model with its default setting. To ensure a fair comparison, we use \llava as the VQA model, aligning with the model used in our experiments.
    \item \textbf{Length-normalized Entropy (LN-Entropy) \citep{malinin2021uncertainty}.} This approach normalizes the joint log-probability of each sequence by dividing it by the sequence length and is proposed by \citep{malinin2021uncertainty} for uncertainty quantification in LLM. Following \citet{kuhnSemanticUncertaintyLinguistic2023}, we also apply multinomial sampling instead of using an ensemble of models.
    \item \textbf{Semantic Entropy \citep{kuhnSemanticUncertaintyLinguistic2023}.} This method introduces the concept of semantic entropy, which measures the uncertainty over different meanings. We implement this method based on their proposed approach of clustering the generated sequences by \texttt{Deberta} as the text entailment model, and then computing entropy based on these clusters. 
    \item \textbf{EigenScore \citep{chenLLMsInternalStates2024}.} We follow their default settings and compute the log determinant of the covariance matrix by Eigenvalues via Singular Value Decomposition (SVD), with the exception of the jitter term value -- we found that using a jitter term of $10^{-8}$ rather than their default setting of $10^{-3}$ improves their performance, hence we applied that and reported the improved performance.  
    \item \textbf{Verbalized Confidence \citep{xiong2024llmsexpressuncertaintyempirical}.} This method is applied specifically to blackbox models where we instruct it to provide a measure of its own confidence. For a single instance, we sample generations $k$ times and return the most frequent answer along with the average reported confidence by the model.
    \item \textbf{Image generation UQ methods }. We implement PUNC \citep{franchi2025towards} as the image generation uncertainty metric. This approach tries to generate the caption from the generated image and compute the text similarity between the new generated caption and the input caption through text similarity metrics such as ROUGE or BertScore. We use their default settings with the \llava as the caption generation Vision-Language model and ROUGE as the text similarity metric. 
\end{itemize}

\subsubsection{Experimental settings}
\begin{itemize}
    \item \textbf{Models and parameters.} We primarily use \llava as our image-text MLLM, with further analysis on other models provided in \Cref{app:models}, \modelphi as audio-text MLLM, and \llavavideo as video-text ones. Following past work \citep{kuhnSemanticUncertaintyLinguistic2023}, for each task instance $\task$, we evaluate the model accuracy $\acc(\model,\task)$ based on the model's response $\hat{y}$ generated via low-temperature sampling (\(T=0.01\)). For the computation of the various uncertainty metrics that require multiple samples, we apply Monte Carlo sampling to generate $k$ samples from the MLLM using $T=1$ and $\text{top\textunderscore p}=0.9$. In the main paper, we use the number of generated samples $k=50$, and ablation results on the impact of this hyperparameter are presented and discussed in \Cref{app:num_gen}.
     \item \textbf{Evaluation.} We evaluate a task instance accuracy $a(\model, \task)$ by comparing the model's low-temperature generated answer against its ground-truth answer. For benchmarks where there are multiple possible ground-truth answers, we evaluate $a(\model, \task)$ over all such answers and take the highest accuracy score (e.g., for exact match, we consider the model as correct as long as its generation matches one of the possible answers). Given a model prediction $\ans_{\task}$ and a reference answer $\ans_{\task}^*$, following prior work \citep{kuhnSemanticUncertaintyLinguistic2023}, we compute the accuracy score using ROUGE-L, exact match, and LLM-as-judge (GPT4o). 
    In the main paper, we report results using exact match, and conducted ablations in \Cref{app:correctness} showing that our results are robust to the choice of these evaluation functions.
    \item \textbf{Blackbox APIs.} For OpenAI's GPT models, we used $n=50$ generations per prompt. For Anthropic's Claude 3.5 Haiku model, we used the same model parameters as specified above but a smaller number of generations $n=20$ due to limitations on API credits.
    \item \textbf{Image/Audio Generation settings.} In this setting, we use NExT-GPT and AnyGPT models as the image/audio generation models, and their default configurations. To make it consistent with MLLM understanding tasks, we also use low-temperature generation for computing CLIP score, and multi-sampling with $n=10$, temperature $T=1$, and  $\text{top\textunderscore p}=0.9$, 
    We conduct experiments on 500 task instances of the MS-COCO caption validation set \citep{capeval2015} for the image generation task, and the full Audiocap test set for the audio generation task. We compare \algname to image generation uncertainty quantification method PUNC \citep{franchi2025towards}. To evaluate the performance of multimodal generation uncertainty methods, instead of introducing the in-distribution and out-of-distribution datasets and trying to let uncertainty metrics classify these sets as in \citep{franchi2025towards}, we show that these uncertainty metrics satisfy \ref{r:r2a} by computing the Pearson Correlation between the uncertainty metric and the negative of quality scores, including continuous CLIP score \citep{hessel-etal-2021-clipscore}, or CLAP score \citep{elizalde2023clap} for image or audio generation tasks, respectively. These quality scores compute the similarity between the generated image/audio and its corresponding input caption from a real image/audio.
    
\end{itemize}

\subsubsection{Prompts}
\label{app:prompts}
Following \citet{liu2023llava}, we use the following prompt for all baseline tasks:

\texttt{
<modality>. Answer this question in a word or a phrase. \{question\} 
}

The prompt used to elicit verbalized confidence from the blackbox API models are slightly different, such that they output their confidence in the answer along with the response. In accordance with \citet{xiong2024llmsexpressuncertaintyempirical}, we use the following prompt to extract verbalized confidence:

\texttt{
<modality>. Read the question, provide your answer, and your confidence in this answer. Note: The confidence indicates how likely you think your answer is true. 
Use the following format to answer:  “‘Answer and Confidence (0-100): [ONLY a word or a phrase; not a complete sentence], [Your confidence level, please only include the numerical number in the range of 0-100]\%”’
Only give me the reply according to this format, don’t give me any other words.  
Now, please answer this question and provide your confidence level.
Question: \{question\} 
}

\subsection{Evaluating TPR under FPR Constraints}
\label{app:tpr}

In addition to the AUROC metric reported in \Cref{sec:auroc}, we also provide results on the True Positive Rate (TPR) achievable for a given False Positive Rate (FPR), which users might have different minimum requirements for based on their application. As shown in \Cref{tab:tpr_results}, we provide True Positive Rate (TPR) at $10\%$ and $1\%$ FPR levels, and the results generally align with AUROC trends reported in the main text, where \algname consistently performs well compared to baselines across datasets.

\begin{table}[h]
\caption{True Positive Rates (TPR) at different False Positive Rate (FPR) thresholds for various uncertainty quantification methods across multimodal tasks. Results are reported at FPR levels of 10\% and 1\%.}
\label{tab:tpr_results}
\centering
\huge
\resizebox{\columnwidth}{!}{
\begin{tabular}{cl|ccccc|cc|c|c}
\toprule
\multirow{2}{*}{Metric} & \multicolumn{1}{c|}{\multirow{2}{*}{Method}} & \multicolumn{5}{c|}{Image} & \multicolumn{2}{c|}{Audio} & \multicolumn{1}{c|}{Video} & \multicolumn{1}{c}{\multirow{2}{*}{Avg}} \\
\cmidrule(lr){3-7} \cmidrule(lr){8-9} \cmidrule(lr){10-10}
 & & VQAv2 & OKVQA & AdVQA & MathVista & VQA-RAD & \makecell{SLUE-P2\\SQA5}  & \makecell{Spoken\\SQuAD} & \makecell{Video-MME\\short}  & \\
\midrule
\midrule
\multirow{5}{*}{\makecell{TPR@10\%\\FPR $\uparrow$}}
 & \cu     & 0.362 & 0.095 & 0.189 & 0.408 & 0.189 & - & - & - & 0.249 \\
 & \lne    & 0.282 & 0.244 & 0.168 & 0.347 & 0.127 & 0.299 & 0.248 & 0.287 & 0.250 \\
 & \se     & 0.574 & 0.327 & 0.419 & 0.437 & 0.511 & \textbf{0.557} & 0.464 & 0.311 & 0.450 \\
 & \eigen  & 0.602 & 0.340 & 0.466 & 0.483 & \textbf{0.601} & 0.443 & 0.435 & 0.534 & 0.488 \\
 & Ours    & \textbf{0.629} & \textbf{0.369} & \textbf{0.477} & \textbf{0.497} & 0.587 & 0.522 & \textbf{0.490} & \textbf{0.539} & \textbf{0.514} \\
\midrule
\multirow{5}{*}{\makecell{TPR@1\%\\FPR $\uparrow$}}
 & \cu     & 0.049 & 0.008 & 0.019 & 0.030 & 0.023 & - & - & - & 0.026 \\
 & \lne    & 0.057 & 0.030 & 0.066 & 0.075 & 0.065 & 0.025 & 0.023 & 0.022 & 0.045 \\
 & \se     & 0.177 & 0.057 & 0.125 & \textbf{0.136} & 0.286 & 0.095 & \textbf{0.169} & 0.054 & 0.137 \\
 & \eigen  & 0.215 & 0.074 & 0.171 & 0.086 & 0.304 & 0.134 & 0.118 & 0.274 & 0.172 \\
 & Ours    & \textbf{0.230} & \textbf{0.091} & \textbf{0.185} & 0.131 & \textbf{0.326} & \textbf{0.154} & 0.162 & \textbf{0.287} & \textbf{0.196} \\
\bottomrule
\end{tabular}%
}
\end{table}

\subsection{Plots for risk-linearity \textbf{\ref{r:r2a}}}
\label{app:cpc_plot}
To better visualize the performance of the various metrics for proportionality \textbf{\ref{r:r2a}}, we plot the error rate ($\mathbb{P}[\acc(\model,\task^*)=0]$) v.s. uncertainty score $u$ on the VQA datasets in \Cref{fig:cpc_plots}. \algname manages to achieve the strongest linear correlation with error rate compared to all other metrics. This satisfies the desiderata of \textbf{\ref{r:r2a}}.

\begin{figure}[h]
    \centering
    \includegraphics[width=\linewidth]{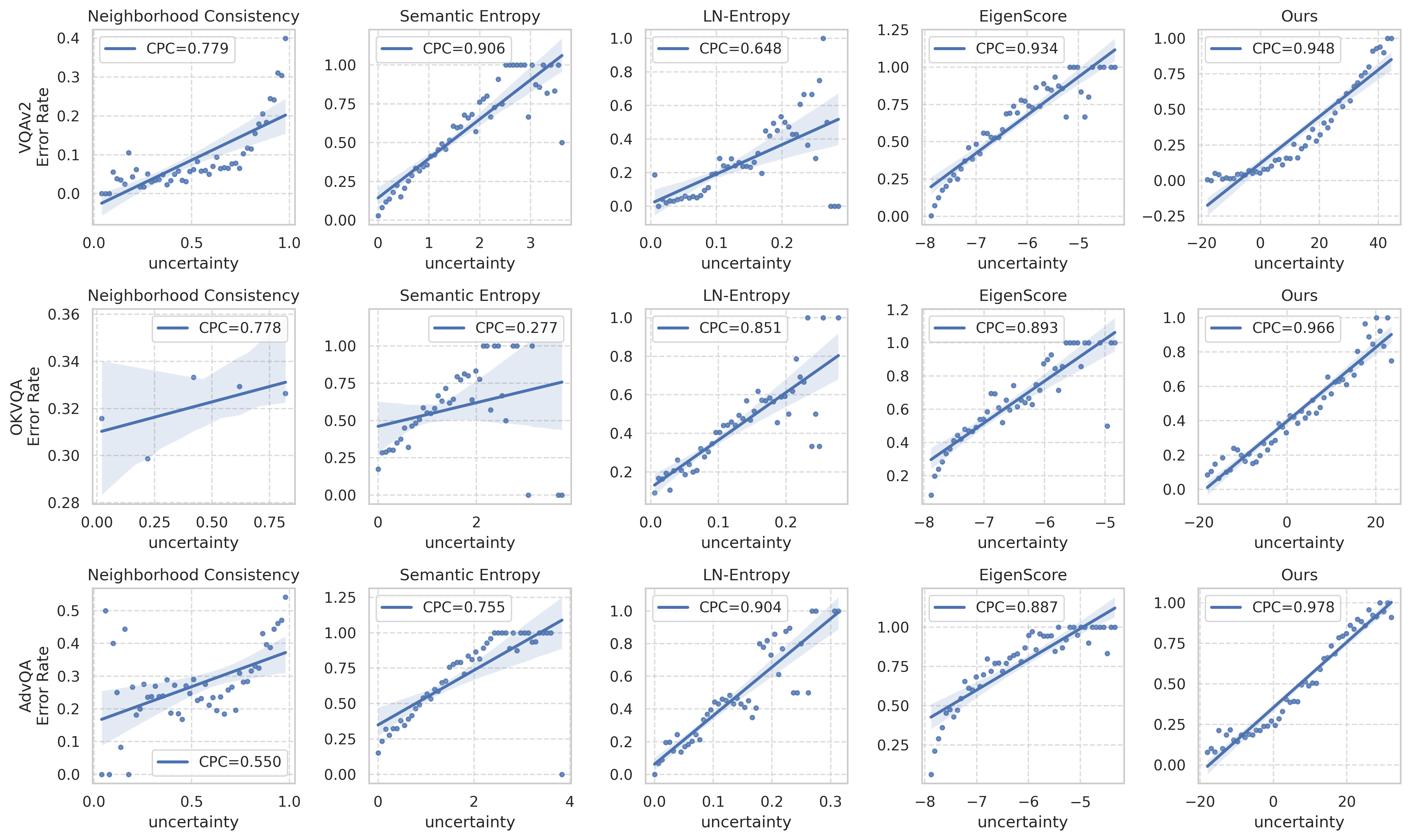}
    \caption{Pearson correlation plots of the uncertainty scores $u$ on the VQAv2, OKVQA, and AdVQA, which demonstrate \algname's strong correlation compared to other metrics.}
    \label{fig:cpc_plots}
\end{figure}

\subsection{Single sampling method}
\label{app:singlesample}

We have also run experiments on basic uncertainty metrics that use only a single MLLM response, rather than a sampled set of MLLM responses. We ran the single-sample methods listed in \citet{xiong2024efficient}: Sequence Probability, Mean Token Entropy \citep{fomicheva2020unsupervised}, and Perplexity.

\Cref{tab:single} shows these methods' results for the various MLLM datasets, along with \algname, based on five response generations. Note that while the single-sampled methods may be cheaper to compute, they also produce significantly worse performance results compared to \algname with $k=5$. The appropriate metric to use would depend on the application requirements. For settings that require better uncertainty metric performance, \algname would likely be a good choice especially since accelerated batched response generation \citep{kwonEfficientMemoryManagement2023} is fast and typically not a computational resource bottleneck, while single-sample methods may be more suitable for very time-sensitive applications.

\subsection{Statistical Tests for AUROC (\ref{r:classify}), CPC (\ref{r:r2a}), ECE (\ref{r:r2b}), and AURAC results}
\label{app:statistical-t-test}
To determine whether the difference in performance between \algname and other baselines is statistically significant, we performed t-tests and Wilcoxon signed-rank tests (non-parametric) to ensure the results agree.

\begin{itemize}
    \item \textbf{$H_0$}: $X_{\algname} - X_{other} = 0$
    \item \textbf{$H_1$}: $X_{\algname} - X_{other} > 0$
\end{itemize}

where $X_{\algname}$ is the set of scores from various metrics using \algname and $X_{other}$ is the set of metric scores from baseline methods, including length-normalized entropy, Eigenscore and semantic entropy. From \Cref{tab:statistical-tests}, both the t-test and the Wilcoxon sign-rank test agree, indicating that the advantage of \algname over other baselines in AURAC, AUROC, CPC and ECE between \algname is statistically significant.

\begin{table}[h]
\caption{Statistical tests calculated based on the results from models: Llava-13B, Llava-7B and Mllama-11B with datasets: VQAv2, AdVQA and OKVQA. The p-value of both tests are $\ll$ 0.01, thus there is sufficient evidence at the 1\% level of significance to conclude that $X_{\algname} - X_{other} > 0$ for the metrics AURAC, CPC, AUROC, and $X_{\algname} - X_{other} < 0$ for ECE.}
\label{tab:statistical-tests}
\centering
\begin{tabular}{l|c|c|c}
\toprule
 & \lne & EigenScore & SE \\
\midrule

\textbf{AUROC (\ref{r:classify})} & & & \\
t-test p-value        
& $1.03 \times 10^{-5}$ 
& $2.25 \times 10^{-8}$ 
& $1.94 \times 10^{-6}$ \\
wilcoxon test p-value 
& $1.95 \times 10^{-3}$ 
& $1.95 \times 10^{-3}$ 
& $1.95 \times 10^{-3}$ \\
\midrule

\textbf{CPC (\ref{r:r2a})} & & & \\
t-test p-value        
& $4.50 \times 10^{-4}$ 
& $9.15 \times 10^{-4}$ 
& $1.06 \times 10^{-2}$ \\
wilcoxon test p-value 
& $1.95 \times 10^{-3}$ 
& $1.95 \times 10^{-3}$ 
& $1.95 \times 10^{-3}$ \\
\midrule

\textbf{ECE (\ref{r:r2b})} & & & \\
t-test p-value        
& $5.91 \times 10^{-3}$ 
& $5.78 \times 10^{-4}$ 
& $2.78 \times 10^{-3}$ \\
wilcoxon test p-value 
& $3.91 \times 10^{-3}$ 
& $1.95 \times 10^{-3}$ 
& $3.91 \times 10^{-3}$ \\

\midrule

\textbf{AURAC} & & & \\
t-test p-value        
& $2.05 \times 10^{-4}$ 
& $3.96 \times 10^{-4}$ 
& $1.71 \times 10^{-5}$ \\
wilcoxon test p-value 
& $1.95 \times 10^{-3}$ 
& $1.95 \times 10^{-3}$ 
& $1.95 \times 10^{-3}$ \\

\bottomrule
\end{tabular}

\end{table}

\subsection{Qualitative Results}
\label{app:qualitative}

We provide some qualitative examples to provide better intuition on how both terms, representing the semantic volume and incoherence scores respectively, are important in \algname. We analyze two distinct failure modes where considering only a single term is sub-optimal: one where considering only semantic diversity fails due to low variance, and another where considering only model probabilities fail due to overconfidence. 

In the first scenario (\Cref{fig:qualitative_U_fail}), the model incorrectly answers "10" to a counting question. Because the sampled hallucinations (e.g., "6", "7", "8") are all digits, they cluster tightly in the embedding space. Consequently, $\unadjmethod$ yield low values, falsely signaling that the model is certain. However, the probability distribution across these tokens is flat. $\quadent$ successfully captures this low predictive confidence, driving the total \algname score up. We further plot the 2D convex hulls of the PCA-reduced embeddings to illustrate how the volume increase -- from $\unadjmethod$ to \algname, which is scaled with incoherence scores.

In the second scenario (\Cref{fig:qualitative_Q_fail}), when asked about a "No Littering" sign, the model incorrectly predicts "Parking" with high probability, resulting in a low $\quadent$ that falsely classifies the prediction as safe. However, the sampled responses, ranging from "Traffic" to "Dump" and "Truck", are semantically distinct. Unlike the clustered digits in the first example, these hallucinations create a large dispersion in the embedding space. The unadjusted semantic volume ($\unadjmethod$) effectively captures this disagreement, penalizing the top-token overconfidence that $\quadent$ missed. 

Together, these examples confirm that $\unadjmethod$ and $\quadent$ cover mutually exclusive blind spots, ensuring robust uncertainty quantification across diverse error types.

\begin{figure}
    \centering
    \includegraphics[width=0.8\linewidth]{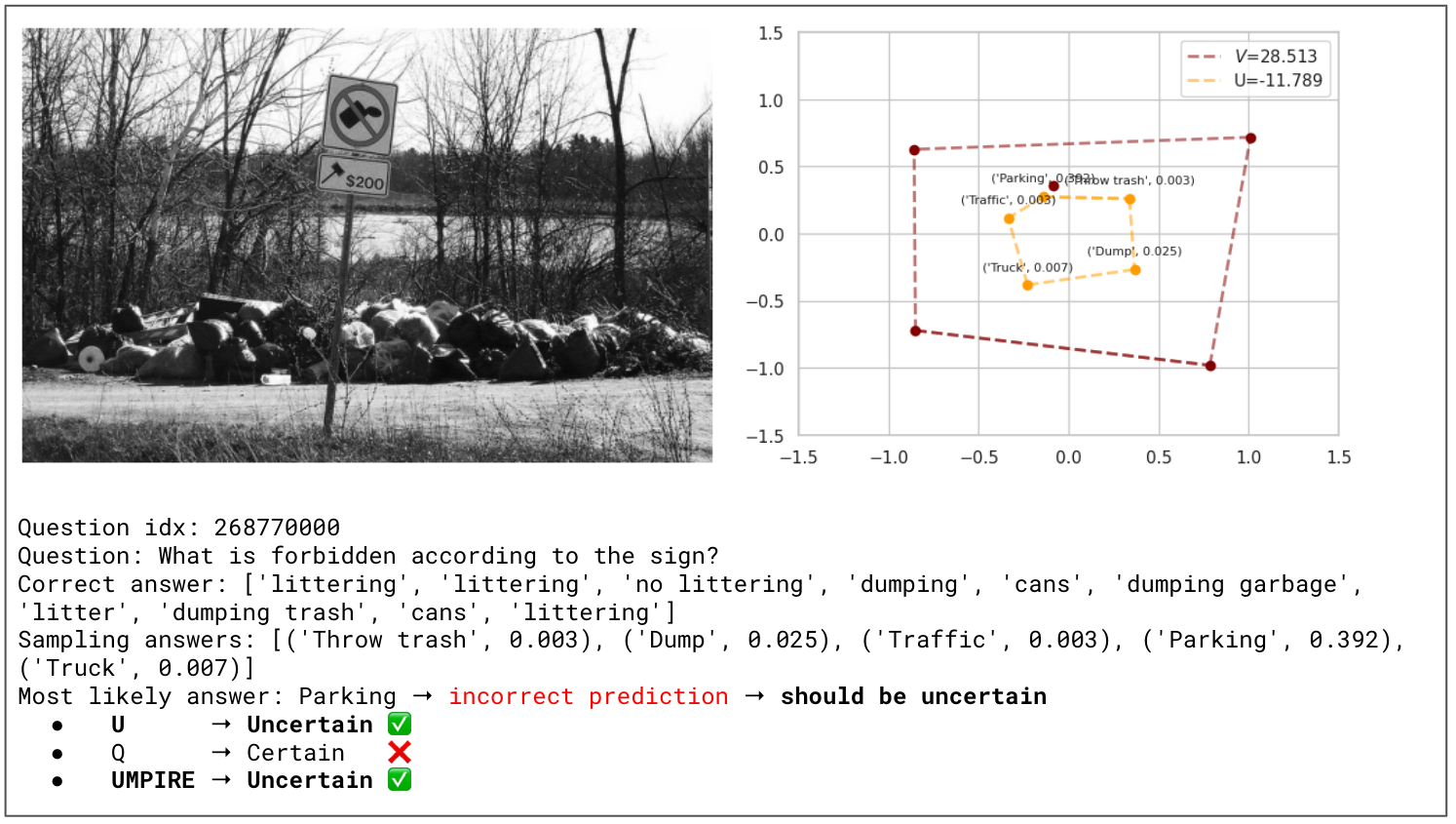}
    \caption{Qualitative example of $\unadjmethod$ correcting model overconfidence. The model incorrectly predicts "Parking" with high confidence, causing the Incoherence score ($\quadent$) to fail (predicting "Certain"). However, because the sampled responses ("Traffic", "Dump", "Truck") are semantically scattered, the high semantic volume ($\unadjmethod$) correctly identifies the confusion, allowing \algname to successfully detect the error.}
    \label{fig:qualitative_Q_fail}
\end{figure}

\begin{figure}
    \centering
    \includegraphics[width=0.8\linewidth]{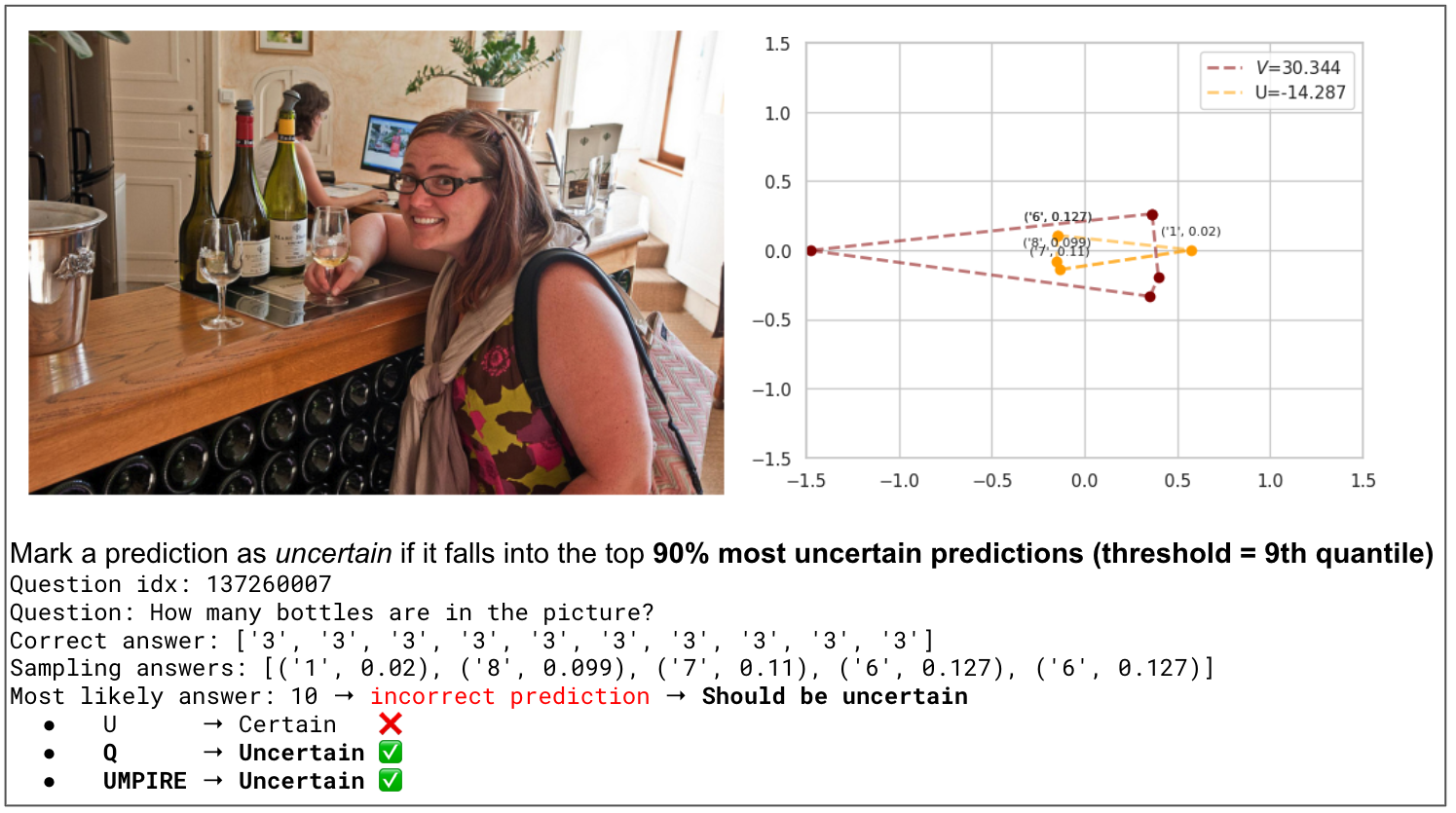}
    \caption{Qualitative example of \algname's robustness on tasks with low semantic variance (e.g., yes/no, counting). While the clustered embeddings of the sampled answers (digits) fool $\unadjmethod$ into predicting certainty, \algname uses the high Incoherence score ($\quadent$) to correctly flag the prediction as uncertain.}
    \label{fig:qualitative_U_fail}
\end{figure}

\begin{table}[h]
\caption{Comparison of the performance of single sampling methods and \algname across various VQA datasets.}
\label{tab:single}
\centering
\begin{tabular}{c|lccc}
\toprule
Dataset & Method & AUROC $\uparrow$ & ECE $\downarrow$ & CPC $\uparrow$ \\
\midrule

\multirow{4}{*}{VQAv2}
& Seq Prob & 0.632 & 0.121 & 0.374 \\
& Mean Token Entropy & 0.628 & 0.129 & 0.046 \\
& Perplexity & 0.629 & 0.131 & 0.125 \\
& Ours (k=5) & \textbf{0.873} & \textbf{0.067} & \textbf{0.923} \\

\midrule
\multirow{4}{*}{AdVQA}
& Seq Prob & 0.595 & 0.303 & 0.372 \\
& Mean Token Entropy & 0.590 & 0.302 & 0.170 \\
& Perplexity & 0.592 & 0.336 & 0.151 \\
& Ours (k=5) & \textbf{0.774} & \textbf{0.055} & \textbf{0.959} \\

\midrule
\multirow{4}{*}{OKVQA}
& Seq Prob & 0.581 & 0.304 & 0.463 \\
& Mean Token Entropy & 0.580 & 0.303 & 0.039 \\
& Perplexity & 0.581 & 0.335 & 0.225 \\
& Ours (k=5) & \textbf{0.740} & \textbf{0.097} & \textbf{0.944} \\

\midrule
\multirow{4}{*}{MathVista}
& Seq Prob & 0.628 & 0.539 & 0.322 \\
& Mean Token Entropy & 0.606 & 0.601 & 0.334 \\
& Perplexity & 0.616 & 0.643 & 0.224 \\
& Ours (k=5) & \textbf{0.791} & \textbf{0.087} & \textbf{0.706} \\

\midrule
\multirow{4}{*}{VQA-RAD}
& Seq Prob & 0.540 & 0.525 & 0.140 \\
& Mean Token Entropy & 0.535 & 0.534 & 0.118 \\
& Perplexity & 0.537 & 0.550 & 0.168 \\
& Ours (k=5) & \textbf{0.806} & \textbf{0.090} & \textbf{0.828} \\

\bottomrule
\end{tabular}

\end{table}

\subsection{Results on additional baselines}
\label{app:addition_baseline}
In addition to the representative baselines evaluated in the main text, we expand our comparison to include several other recent UQ methods: VL-Uncertainty (VL-Unc.) \citep{zhangVLUncertaintyDetectingHallucination2024}, which adapts \se \citep{kuhnSemanticUncertaintyLinguistic2023} specifically for vision-language tasks; Verbalized Confidence (Verb. Conf.) \citep{xiong2024llmsexpressuncertaintyempirical}; and three methods from \citep{liReferencefreeHallucinationDetection2024}: Average Probability (AvgProb), Average Entropy (AvgEnt), and SUQ (a training-dependent approach). We conduct these experiments on the full OKVQA validation set, as well as 3,000-sample subsets of the VQAv2 and AdVQA validation sets, followed the same main experimental setting. As shown in \cref{tab:more_baselines}, \algname consistently outperforms these additional baselines in both discrimination and risk-score quality.

\begin{table}[h]
\caption{Additional baseline results on the VQAv2 (3000), OKVQA, and AdVQA (3000) validation sets. The evaluation covers AUROC ($\times 100$) ($\uparrow$ better), CPC ($\times 100$) ($\uparrow$ better), and ECE ($\downarrow$ better). \algname consistently achieves the best performance across all datasets.}
\label{tab:more_baselines}
\centering
\begin{tabular}{ll ccc >{\columncolor[gray]{0.9}}c}
\toprule
\textbf{Metric} & \textbf{Method} & VQAv2 & OKVQA & AdVQA & \textbf{Avg.} \\
\midrule

\multirow{6}{*}{\rotatebox[origin=c]{90}{\textbf{AUROC $\uparrow$}}} 
 & VL-Unc. \citep{zhangVLUncertaintyDetectingHallucination2024} & \underline{80.7} & \underline{67.4} & \underline{74.9} & \underline{74.3} \\
 & AvgProb \citep{liReferencefreeHallucinationDetection2024} & 62.9 & 58.1 & 59.2 & 60.1 \\
 & AvgEnt \citep{liReferencefreeHallucinationDetection2024} & 62.8 & 58.0 & 59.0 & 59.9 \\
 & SUQ \citep{liReferencefreeHallucinationDetection2024} & 52.9 & 49.9 & 61.4 & 54.7 \\
 & Verb.Conf \citep{xiong2024llmsexpressuncertaintyempirical} & 50.5 & 52.7 & 48.4 & 50.5 \\
 & Ours & \textbf{88.2} & \textbf{75.5} & \textbf{78.7} & \textbf{80.9} \\

\midrule

\multirow{6}{*}{\rotatebox[origin=c]{90}{\textbf{CPC $\uparrow$}}} 
 & VL-Unc. \citep{zhangVLUncertaintyDetectingHallucination2024} & \underline{90.2} & \underline{74.2} & \underline{96.7} & \underline{87.0} \\
 & AvgProb \citep{liReferencefreeHallucinationDetection2024} & 12.5 & 22.5 & 15.1 & 16.7 \\
 & AvgEnt \citep{liReferencefreeHallucinationDetection2024} & 4.6 & 3.9 & 17.0 & 8.5 \\
 & SUQ \citep{liReferencefreeHallucinationDetection2024} & 23.1 & 8.1 & 58.7 & 30.0 \\
 & Verb.Conf. \citep{xiong2024llmsexpressuncertaintyempirical} & 17.2 & 18.7 & 12.1 & 16.0 \\
 & Ours & \textbf{94.6} & \textbf{96.6} & \textbf{97.9} & \textbf{95.4} \\

\midrule

\multirow{6}{*} {\rotatebox[origin=c]{90}{\textbf{ECE $\downarrow$}}}
 & VL-Unc. \citep{zhangVLUncertaintyDetectingHallucination2024} & \underline{.040} & \underline{.087} & \underline{.078} & \underline{.068} \\
 & AvgProb \citep{liReferencefreeHallucinationDetection2024} & .131 & .335 & .336 & .267 \\
 & AvgEnt \citep{liReferencefreeHallucinationDetection2024} & .129 & .304 & .302 & .245 \\
 & SUQ \citep{liReferencefreeHallucinationDetection2024} & .180 & .375 & .276 & .277 \\
 & Verb.Conf. \citep{xiong2024llmsexpressuncertaintyempirical} & .214 & .337 & .316 & .289 \\
 & Ours & \textbf{.038} & \textbf{.036} & \textbf{.042} & \textbf{.034} \\

\bottomrule
\end{tabular}
\end{table}

\section{Ablation studies}
\label{app:ablation_study}

\subsection{Embedding layer selection}
\label{app:layers}
We analyzed the impact of the layer index when extracting the embedding vectors by computing the AUROC performance on different embedding matrices extracted from different layer indices, all computed on the first 3000 samples of the VQAv2 validation set. As shown in \Cref{fig:layer_index}, the change in which MLLM layer to use makes the AUROC performance vary only slightly. The last layer still yields the best performance, so we adopt it for all of our experiments.

\begin{figure}[h]
    \centering
    \includegraphics[width=0.5\columnwidth]{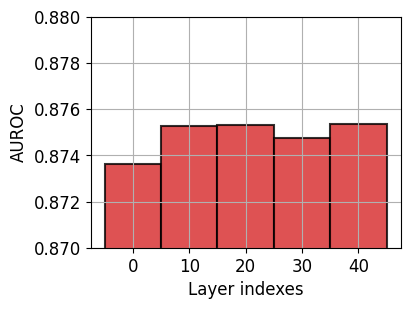}
    \caption{Ablation study on choosing the layer index to extract embedding vectors. Results show that different layer indices only have slight variations in the AUROC performance. 
    }
    \label{fig:layer_index}
\end{figure}

\subsection{Number of generations analysis}
\label{app:num_gen}

To analyze the impact of the number of generations on the various metrics' performance, we conduct an ablation study by varying the number of generated responses (from 2 to 50) per task instance for the first 3000 samples of the VQAv2 validation set. As shown in \Cref{fig:run_time_num_gen}(b), increasing the number of generations generally improves AUROC across all methods, and \algname achieves higher performance with significantly fewer generations compared to baselines. This indicates that our method is more efficient, requiring fewer samples to reach strong performance, whereas other methods continue to rely on additional generations for improvement. The results highlight the robustness of our approach in capturing correctness signals effectively, even with a limited number of generations.

\subsection{Ablation on evaluation parameters}
\label{app:eval_ablation}

\begin{figure}[h]
    \centering
    \includegraphics[width=\columnwidth]{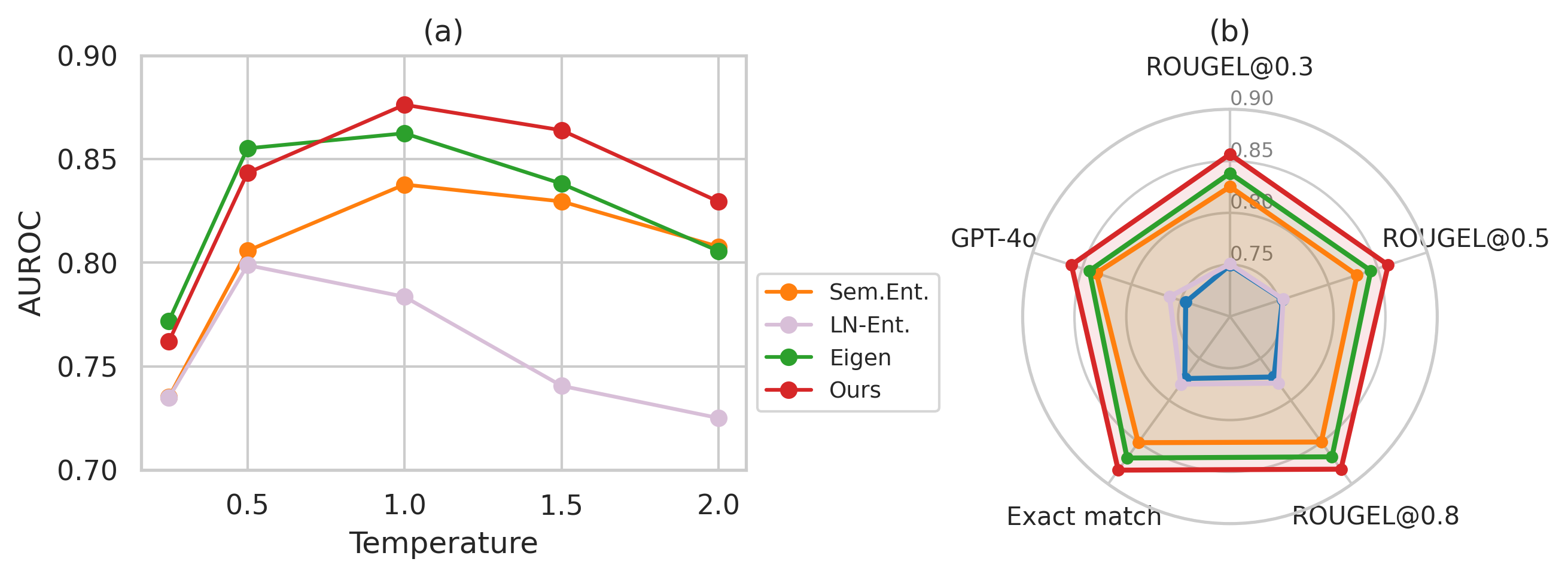}
    \caption{Ablation study on the (a) Impact of temperature during the generation process on evaluation performance. (b) Evaluation methods. \algname consistently outperforms baseline approaches regardless of the chosen evaluation functions. }
    \label{fig:temp_evaluator}
\end{figure}

\paragraph{Evaluation function $\acc(\model,\task^*)$} \label{app:correctness}Following the setting in \citep{kuhnSemanticUncertaintyLinguistic2023}, we further evaluate the performance of our method and baselines
under exact match, various levels of the ROUGE-L and LLM-as-judge with GPT4o. \Cref{fig:temp_evaluator}(b) presents the AUROC scores across different evaluation functions $\acc(\model,\task^*)$ on the first 3000 samples of the VQAv2 validation set, demonstrating that our method consistently outperforms baseline approaches regardless of the chosen evaluation functions. 
These results highlight the versatility and robustness of our approach across different correctness evaluation criteria.

\begin{figure}
    \centering
    \includegraphics[width=\linewidth]{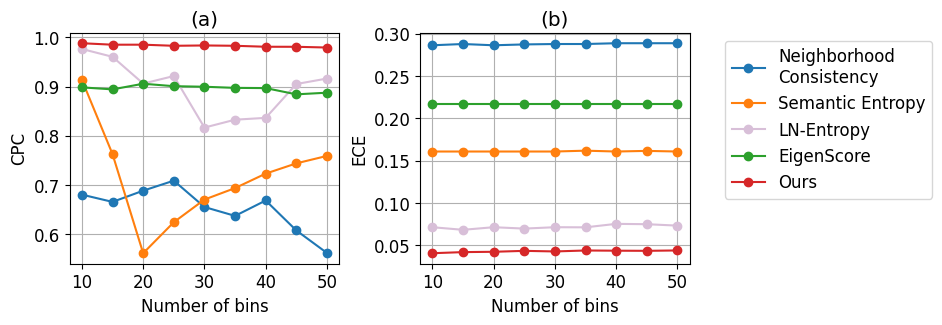}
    \caption{Results for the effect of number of bins on (a) CPC and (b) ECE. Both measures show that \algname consistently outperforms baselines.}
    \label{fig:evaluator_and_num_bin}
\end{figure}

\paragraph{Effect of number of bins in ECE and CPC.} \label{app:num_bins} In the main paper, we mainly use the same number of bins as $50$ for CPC and ECE. We further analyzed the effect of the number of bins when computing ECE and CPC by evaluating the first 3000 samples of the AdVQA dataset. \Cref{fig:evaluator_and_num_bin} illustrates that \algname still achieves the best and consistent performance across all bin value settings.

\subsection{Sampling temperature}
\label{app:sampling}

Besides the number of generations in \Cref{app:num_gen}, we analyzed the impact of temperature during the generation process on the evaluation performance. We conducted an ablation study by varying the generation temperature (from 0.25 to 2) on the first 3000 samples of the VQAv2 validation set. Note that the temperature used for sampling responses to compute $\metric(\model, \task)$ need not be the same as that used to generate MLLMs' final output $\hat{\ans}$ to task instances. Higher temperatures settings could be used to sample from the MLLM for assessing uncertainty, while the final MLLM output could be obtained low finite temperature sampling. As shown in \Cref{fig:temp_evaluator}, most baselines tend to perform optimally around a sampling temperature of 1, with \algname outperforming the best performance of other baselines.

\subsection{Model sizes and families analysis}
\label{app:models}
We analyze the impact of model size and architecture family on evaluation performance by comparing different models across various sizes and families on the first 3000 samples of the VQAv2 validation set for the image-text understanding task and the SLUE-P2-SQA5 test set for the audio-text ones. As shown in \Cref{fig:model_abl}, we observe a slight increase in AUROC as the model size increases within the same family. This suggests that larger models tend to generate more informative and reliable outputs. Additionally, our method show a strong performance AUROC across all tested models, demonstrating its robustness regardless of model size or architecture. These findings highlight that while larger models can enhance performance, our approach remains effective across different model scales and families.

\begin{figure}[H]
    \centering
    \includegraphics[width=0.8\columnwidth]{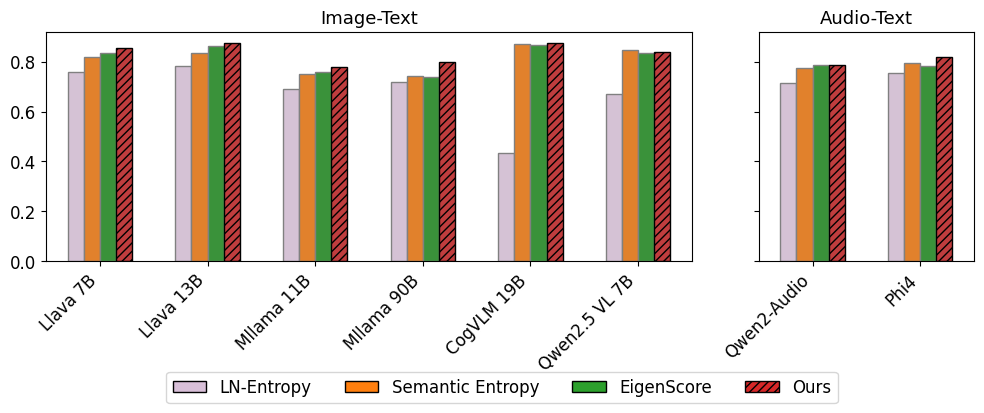}
    \caption{Ablation study across different models on image-text (VQAv2) and audio-text (SLUE-P2-SQA5) understanding datasets, evaluating AUROC performance for \lne, \se, \eigen, and \algname. The results indicate that \algname consistently achieves strong AUROC across various models, including \texttt{Llava-v1.5-7b, Llava-v1.5-13b, Llama-3.2-11B-Vision, Llama-3.2-90B-Vision}, \texttt{cogvlm2-llama3-chat-19B}, \texttt{Qwen2.5-VL-7B-Instruct} for image-text, and \texttt{Qwen2-Audio-7B}, \texttt{Phi4} for audio-text. This highlights the robustness and effectiveness of our approach across different model architectures.}
    \label{fig:model_abl}
\end{figure}

\subsection{Results of other blackbox and whitebox proxy models in blackbox settings.}
\label{app:blackbox}
\subsubsection{Blackbox models}
\begin{figure*}[h]
    \centering
    \includegraphics[width=\linewidth]{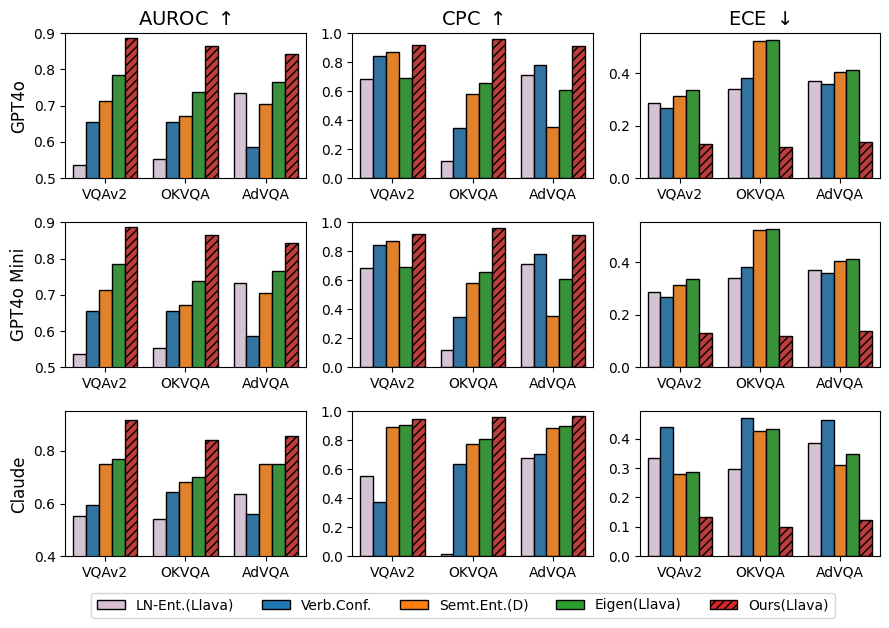}
    \caption{\algname metric consistently outperforms other baselines across various black-box models, including GPT4o, GPT4o-mini, and Claude 3.5 Haiku.}
    \label{fig:blackbox_all}
\end{figure*}

As in \Cref{fig:blackbox_all}, we find that \algname also outperforms other baselines using different black-box models, including Claude 3.5 Haiku \citep{TheC3}, GPT4o-mini \citep{openai2024gpt4omini}.

\subsubsection{White-box proxy models}
The experiments in \Cref{sec:blackbox} use a simple approach of applying a vanilla \llava proxy model for all blackbox models. As seen in our empirical results \Cref{fig:blackbox}, \algname consistently outperforms baselines without any fine-tuning of the proxy model. 

\Cref{tab:diff_proxy} shows that \algname with different whitebox proxy models for GPT-4o also achieves good performance
on the first 3000 samples of the VQAv2 validation set. We analyze the effect of the proxy choice in more detail in \Cref{app:proxy_mismatch}.

\begin{table*}[h]
\caption{Results of \algname in blackbox settings with different proxy models.}
\label{tab:diff_proxy}
\centering
\begin{tabular}{l|ccc}
\toprule
\textbf{Method} & \textbf{AUROC $\uparrow$} & \textbf{CPC $\uparrow$} & \textbf{ECE $\downarrow$} \\
\midrule
\midrule
\algname (\texttt{Llava-v1.5-13b})   & 0.890 & 0.904 & 0.094 \\
\algname (\texttt{Llava-v1.5-7b})    & 0.893 & 0.900 & 0.087 \\
\algname (\texttt{Llama-3.2-11B-Vision})  & 0.839 & 0.943 & 0.091 \\
\bottomrule
\end{tabular}

\end{table*}

\subsubsection{Proxy-target mismatch}
\label{app:proxy_mismatch}
To better understand when the whitebox proxy approach works well, we fix the target model to \llava (VQAv2, $3{,}000$ questions, $50$ generations each) and compute \algname on the same target responses using six proxy models. These proxies differ from the target in model size, language backbone, image encoding, and model family (\Cref{tab:proxy}).

\begin{table}[h]
\centering\small
\caption{\algname{} under proxy-target mismatch (target: \llava). ECE (Platt) uses a 20\% labeled calibration split as in \Cref{app:platt}.}
\label{tab:proxy}
\begin{tabular}{llccc}
\toprule
Proxy & Mismatch & AUROC & ECE (min-max) & ECE (Platt) \\
\midrule
Target (self) & -- & .876 & .026 & .015 \\
LLaVA-1.5-7B & size & .875 & .028 & .022 \\
LLaVA-1.6-Vicuna-13B & image encoding & .840 & .171 & .052 \\
BakLLaVA-v1 & size + backbone & .840 & .206 & .040 \\
LLaVA-1.6-Mistral-7B & enc.\ + size + backbone & .845 & .156 & .050 \\
Qwen2.5-VL-7B & full family & .863 & .112 & .030 \\
Llama-3.2-11B-Vision & full family & .843 & .130 & .033 \\
\bottomrule
\end{tabular}
\end{table}

The discrimination (\ref{r:classify}) performance of \algname is robust to the choice of proxy. AUROC stays within $0.840$ to $0.875$ across all proxies, compared to $0.876$ when the target model computes its own score. The risk-score quality (\ref{r:risk}) is more sensitive to the proxy. The unlabeled ECE (\ref{r:r2b}) increases from $0.026$ to as much as $0.206$ when the proxy differs in image encoding, backbone, or model family, although Platt scaling with a small labeled set reduces it to between $0.022$ and $0.052$. We explain this behavior through the decomposition of \algname in \Cref{tab:proxy_uq}.

\begin{table}[h]
\centering\small
\caption{$\unadjmethod$ and $\quadent$ under proxy-target mismatch.}
\label{tab:proxy_uq}
\begin{tabular}{lcccc}
\toprule
& \multicolumn{2}{c}{$\unadjmethod$} & \multicolumn{2}{c}{$\quadent$} \\
\cmidrule(lr){2-3}\cmidrule(lr){4-5}
Proxy & AUROC & ECE & AUROC & ECE \\
\midrule
Target (self) & .863 & .038 & .876 & .065 \\
LLaVA-1.5-7B & .859 & .039 & .875 & .080 \\
LLaVA-1.6-Vicuna-13B & .863 & .040 & .759 & .268 \\
BakLLaVA-v1 & .863 & .044 & .724 & .560 \\
LLaVA-1.6-Mistral-7B & .862 & .040 & .759 & .280 \\
Qwen2.5-VL-7B & .861 & .039 & .836 & .297 \\
Llama-3.2-11B-Vision & .860 & .042 & .803 & .379 \\
\bottomrule
\end{tabular}
\end{table}

As shown in \Cref{tab:proxy_uq}, $\unadjmethod$ computed from proxy embeddings is stable across proxies (AUROC $0.859$ to $0.863$). This is consistent with existing evidence that capable models with different architectures learn similar representations \citep{platonic}. In contrast, the performance of $\quadent$ varies substantially across proxies. When responses are sampled from the target distribution $\mathbb{P}(\cdot|\query_t)$ but scored by a proxy distribution $\mathbb{P}'(\cdot|\query_t)$, we have
\begin{equation*}
\mathbb{E}_{Y\sim \mathbb{P}}\big[1-\mathbb{P}'(Y|\query_t)\big] = 1-\sum_y \mathbb{P}(y|\query_t)\,\mathbb{P}'(y|\query_t) = \Pr(Y\neq Y'),\quad Y\sim \mathbb{P},\; Y'\sim \mathbb{P}' \text{ independent},
\end{equation*}
Hence, $\quadent$ also captures the disagreement between the target and proxy response distributions, and becomes less reliable as their likelihoods diverge.
When the proxy and target likelihoods are compatible (e.g., the target itself or a smaller model from the same family), combining both terms still improves over $\unadjmethod$ alone. Overall, proxy-based \algname provides robust discrimination (\ref{r:classify}), and a small labeled calibration set can be used to improve its risk-score quality (\ref{r:risk}) when the proxy and target models differ substantially.

\subsection{Length-normalized effect}
\label{sec:ln_prob}
In prior work on uncertainty estimation and related scoring functions, length normalization has often been applied to adjust for biases introduced by varying response lengths \citep{kuhnSemanticUncertaintyLinguistic2023, malinin2021uncertainty} and handle small output probability values for long token generations. Motivated by this, we explored whether applying token probability length normalization to the incoherence term would significantly affect the \algname metric performance. Empirically, as shown in \Cref{tab:incoherence-ln}, we observed that applying length normalization does not cause major variations in performance compared to the non-length normalized version of \algname. Depending on the task and benchmark being evaluated on, length normalization might yield minor improvements or degradation of performance across AUROC, CPC ECE, with no consistent trend.

\paragraph{Recommendation and theoretical implications.} As discussed in \Cref{sec:prac_consi}, we recommend using raw probabilities $p_i$ for tasks with short answers, and length-normalized probabilities $\bar p_i := p_i^{1/N_i}$ for tasks with long answers, where raw sequence probabilities become very small. Since $\bar p_i\in(0,1]$, $\quadent$ remains a bounded and unbiased Monte Carlo estimate of $\mathbb{E}[1-\bar p(Y)]$. Hence, the decomposition (\cref{prop:app:decomp}), the results on $\unadjmethod$ (\cref{lem:app:detlemma,prop:app:psdmono}), and the concentration bound (\cref{prop:app:concentration}) continue to hold, with the ranking result now based on the expected-score gap of the length-normalized metric.
Length normalization only changes the interpretation of $\quadent$ as quadratic entropy and collision probability (\Cref{lem:app:QasH2}) and the results that build on it (\Cref{lem:app:coarsen,cor:app:dominant,cor:app:manymodes}), since $\bar p$ is no longer the probability under which responses are sampled. 
Empirically, length-normalized \algname remains effective, especially in long-form settings (\Cref{tab:long-form}).

\paragraph{Long-form generation tasks.} \label{app:long-form} To align with past works \citep{kuhnSemanticUncertaintyLinguistic2023, nikitinKernelLanguageEntropy2024, khan2024consistency}, we conduct extensive experiments and report the performance of \algname and baselines on the short-form generation task (open-ended question answering). In this section, we would like to further provide the performance of \algname and baselines on the long-form generation setting via the visual chain-of-thought (CoT) reasoning task. Particularly, we evaluate \algname and baselines on the MMMU--Computer Science test set and MathVista testmini set using the \texttt{Qwen3VL-thinking-8B model}, generating an average response length of $1740$ tokens. As shown in \cref{tab:long-form}, \algname with length normalization remains robust and performant for the long-form generation task.

\begin{table}[h]
\caption{\algname{} with raw sequence probabilities $p_i$ versus length-normalized probabilities $\bar p_i$ on short-answer VQA datasets. Length normalization causes only minor, inconsistent variations in this regime.}
\label{tab:incoherence-ln}
\centering
\resizebox{\columnwidth}{!}{%
\begin{tabular}{llccccc}
\toprule
Metric & Method & VQAv2 & AdVQA & OKVQA & MathVista & VQA-Rad \\
\midrule
\midrule
\multirow{2}{*}{AUROC ↑} & Without Length Normalized & \textbf{0.882} & \textbf{0.787} & 0.755 & 0.822 & \textbf{0.802} \\
                         & Length Normalized         & 0.875          & 0.779          & \textbf{0.756} & \textbf{0.825} & 0.792 \\
\midrule
\multirow{2}{*}{CPC ↑}   & Without Length Normalized & 0.946          & \textbf{0.979} & \textbf{0.966} & \textbf{0.945} & 0.908 \\
                         & Length Normalized         & \textbf{0.986} & 0.978          & 0.946          & 0.936          & \textbf{0.935} \\
\midrule
\multirow{2}{*}{ECE ↓}   & Without Length Normalized & \textbf{0.038} & 0.042          & 0.036          & 0.071          & \textbf{0.067} \\
                         & Length Normalized         & 0.062          & \textbf{0.019} & \textbf{0.034} & \textbf{0.056} & 0.068 \\
\bottomrule
\end{tabular}%
}

\end{table}

\begin{table}[h]
\caption{Performance of \algname and other baselines on the long-form generation task (visual CoT reasoning on MMMU--Computer Science test set and MathVista testmini set). \algname with length normalization achieves the best performance across all methods. (LN denotes length normalization)}
\label{tab:long-form}
\centering
\begin{tabular}{ll cc}
\toprule
\textbf{Metric} & \textbf{Method} & \textbf{MathVista (CoT)} & \textbf{MMMU-CS (CoT)} \\
\midrule

\multirow{4}{*}{\rotatebox[origin=c]{90}{\textbf{AUROC $\uparrow$}}} 
 & \se (LN) & 71.9 & 65.9 \\
 & \lne & 70.2 & 66.5 \\
 & \eigen & \underline{72.4} & \underline{72.5} \\
 & \algname (LN)& \textbf{74.5} & \textbf{74.0} \\

\midrule

\multirow{4}{*}{\rotatebox[origin=c]{90}{\textbf{CPC $\uparrow$}}} 
 & \se (LN) & 74.6 & 6.6 \\
 & \lne & \underline{84.7} & 56.1 \\
 & \eigen & 62.9 & \underline{67.8} \\
 & \algname (LN) & \textbf{88.0} & \textbf{88.7} \\

\midrule

\multirow{4}{*} {\rotatebox[origin=c]{90}{\textbf{ECE $\downarrow$}}}
 & \se (LN) & .146 & .232 \\
 & \lne & \underline{.061} & .140 \\
 & \eigen & .077 & \underline{.069} \\
 & \algname (LN) & \textbf{.058} & \textbf{.061} \\

\bottomrule
\end{tabular}
\end{table}

\newpage
\section*{NeurIPS Paper Checklist}

\begin{enumerate}

\item {\bf Claims}
    \item[] Question: Do the main claims made in the abstract and introduction accurately reflect the paper's contributions and scope?
    \item[] Answer: \answerYes{} %
    \item[] Justification: The main claims made in the abstract and introduction accurately reflect the paper's contributions and scope.

\item {\bf Limitations}
    \item[] Question: Does the paper discuss the limitations of the work performed by the authors?
    \item[] Answer: \answerYes{} %
    \item[] Justification: The paper discusses the limitations of the work in conclusion section.

\item {\bf Theory assumptions and proofs}
    \item[] Question: For each theoretical result, does the paper provide the full set of assumptions and a complete (and correct) proof?
    \item[] Answer: \answerYes{} %
    \item[] Justification: For each theoretical result, the paper provides the full set of assumptions and a complete (and correct) proof

    \item {\bf Experimental result reproducibility}
    \item[] Question: Does the paper fully disclose all the information needed to reproduce the main experimental results of the paper to the extent that it affects the main claims and/or conclusions of the paper (regardless of whether the code and data are provided or not)?
    \item[] Answer: \answerYes{} %
    \item[] Justification: The paper fully discloses all the information needed to reproduce the main experimental results of the paper to the extent that it affects the main claims and/or conclusions of the paper

\item {\bf Open access to data and code}
    \item[] Question: Does the paper provide open access to the data and code, with sufficient instructions to faithfully reproduce the main experimental results, as described in supplemental material?
    \item[] Answer: \answerYes{} %
    \item[] Justification: The paper provides open access to the data and code, with sufficient instructions to faithfully reproduce the main experimental results, as described in supplemental material.

\item {\bf Experimental setting/details}
    \item[] Question: Does the paper specify all the training and test details (e.g., data splits, hyperparameters, how they were chosen, type of optimizer) necessary to understand the results?
    \item[] Answer: \answerYes{}{} %
    \item[] Justification: The paper specifies all the training and test details (e.g., data splits, hyperparameters, how they were chosen, type of optimizer) necessary to understand the results.

\item {\bf Experiment statistical significance}
    \item[] Question: Does the paper report error bars suitably and correctly defined or other appropriate information about the statistical significance of the experiments?
    \item[] Answer: \answerYes{} %
    \item[] Justification: The paper reports error bars suitably and correctly defined and other appropriate information about the statistical significance of the experiments.

\item {\bf Experiments compute resources}
    \item[] Question: For each experiment, does the paper provide sufficient information on the computer resources (type of compute workers, memory, time of execution) needed to reproduce the experiments?
    \item[] Answer: \answerYes{} 
    \item[] Justification: We report the used GPU and computational time for each uncertainty metric.
    
\item {\bf Code of ethics}
    \item[] Question: Does the research conducted in the paper conform, in every respect, with the NeurIPS Code of Ethics \url{https://neurips.cc/public/EthicsGuidelines}?
    \item[] Answer: \answerYes{} %
    \item[] Justification: The research conducted in the paper conform, in every respect, with the NeurIPS Code of Ethics.

\item {\bf Broader impacts}
    \item[] Question: Does the paper discuss both potential positive societal impacts and negative societal impacts of the work performed?
    \item[] Answer: \answerNA{} %
    \item[] Justification: The paper tries to develop an efficient and effective reliability indicator for MLLM prediction, thus there is no negative societal impact regarding Safety, Security, Discrimination, Surveillance, Deception and Harassment, Environment, Human Rights, Bias, and fairness.
    
\item {\bf Safeguards}
    \item[] Question: Does the paper describe safeguards that have been put in place for responsible release of data or models that have a high risk for misuse (e.g., pre-trained language models, image generators, or scraped datasets)?
    \item[] Answer: \answerNA{} %
    \item[] Justification: The paper poses no such risks.

\item {\bf Licenses for existing assets}
    \item[] Question: Are the creators or original owners of assets (e.g., code, data, models), used in the paper, properly credited and are the license and terms of use explicitly mentioned and properly respected?
    \item[] Answer: \answerYes{} %
    \item[] Justification: The paper cites all existing assets properly.

\item {\bf New assets}
    \item[] Question: Are new assets introduced in the paper well documented and is the documentation provided alongside the assets?
    \item[] Answer: \answerNA{} %
    \item[] Justification: The paper does not release new assets.

\item {\bf Crowdsourcing and research with human subjects}
    \item[] Question: For crowdsourcing experiments and research with human subjects, does the paper include the full text of instructions given to participants and screenshots, if applicable, as well as details about compensation (if any)? 
    \item[] Answer: \answerNA{} %
    \item[] Justification: The paper does not involve crowdsourcing nor research with human subjects.

\item {\bf Institutional review board (IRB) approvals or equivalent for research with human subjects}
    \item[] Question: Does the paper describe potential risks incurred by study participants, whether such risks were disclosed to the subjects, and whether Institutional Review Board (IRB) approvals (or an equivalent approval/review based on the requirements of your country or institution) were obtained?
    \item[] Answer: \answerNA{} %
    \item[] Justification: The paper does not involve crowdsourcing nor research with human subjects.

\item {\bf Declaration of LLM usage}
    \item[] Question: Does the paper describe the usage of LLMs if it is an important, original, or non-standard component of the core methods in this research? Note that, if the LLM is used only for writing, editing, or formatting purposes and does \emph{not} impact the core methodology, scientific rigor, or originality of the research, declaration is not required.
    \item[] Answer: \answerNA{} %
    \item[] Justification: The core method development in this research does not involve LLMs as any important, original, or non-standard components.

\end{enumerate}

\end{document}